\def\year{2016}\relax
\pdfoutput=1
\documentclass[letterpaper]{article}
\usepackage{aaai16}
\usepackage{times}
\usepackage{helvet}
\usepackage{courier}
\usepackage{amssymb}
\usepackage{amsmath}
\usepackage{color}

\usepackage{amsfonts}

\usepackage{mathtools}
\DeclarePairedDelimiter\ceil{\lceil}{\rceil}

\usepackage{graphicx}
\usepackage{caption}
\usepackage{subcaption}

\usepackage{algorithm}
\usepackage[noend]{algorithmic}
\DeclareMathOperator*{\argmin}{arg\,min}
\usepackage{breqn}
\usepackage{bbm}

\usepackage[toc,page]{appendix}

\usepackage{float}
\usepackage{makeidx}

\usepackage{eqparbox}

\usepackage{enumitem}

\usepackage{amsthm}
\theoremstyle{definition}
\newtheorem{theorem}{Theorem}
\newtheorem{lemma}{Lemma}
\newtheorem{proposition}{Proposition}
\newtheorem{corollary}{Corollary}
\newtheorem{defn}{Definition}
\newtheorem{remark}{Remark}
\newtheorem{assumption}{Assumption}
\usepackage{cleveref}
\usepackage{cases}
\usepackage{breqn}

\begin{document}

%
\title{Bounded Optimal Exploration in MDP}
\author{Kenji Kawaguchi\\
Massachusetts Institute of Technology\\
Cambridge, MA, 02139\\
kawaguch@mit.edu\\
}
\maketitle
\begin{abstract}
\begin{quote}
Within the framework of probably approximately correct Markov decision processes (PAC-MDP), much theoretical work has focused on  methods to attain near optimality after a relatively long period of learning and exploration. However, practical concerns require the attainment of satisfactory behavior within a short period of time. In this paper, we relax the PAC-MDP conditions to reconcile theoretically driven exploration methods and practical needs. We propose simple algorithms for discrete and continuous state spaces, and illustrate the benefits of our proposed relaxation via theoretical analyses and numerical examples. Our algorithms also maintain anytime error bounds and average loss bounds. Our approach accommodates both Bayesian and non-Bayesian methods.
\end{quote}
\end{abstract}
\section{Introduction}

\noindent The formulation of sequential decision making as a Markov decision process (MDP) has been successfully applied to a number of real-world problems. MDPs provide the ability to design adaptable agents that can operate effectively in uncertain environments. In many situations, the environment  we wish to model has  unknown aspects, and thus the agent needs to learn an MDP by interacting with the environment. In other words,  the agent  has to \textit{explore} the unknown aspects of the environment to learn the MDP.  A considerable amount of theoretical work on MDPs has focused on efficient exploration, and a number of principled methods have been derived with the aim of learning an MDP to obtain a near-optimal policy. For example, \citeauthor{kearns2002near} \shortcite{kearns2002near} and \citeauthor{strehl2008analysis} \shortcite{strehl2008analysis}  considered discrete state spaces, whereas \citeauthor{bernstein2010adaptive} \shortcite{bernstein2010adaptive} and \citeauthor{pazis2013pac} \shortcite{pazis2013pac} examined continuous state spaces.

In practice, however, heuristics are still commonly used \cite{li2012sample}. The focus of theoretical work (learning a near-optimal policy within a polynomial yet long time) has apparently diverged from practical needs (learning a satisfactory policy within a reasonable time). In this paper, we modify the prevalent theoretical approach to develop theoretically driven methods that come close to practical needs.

\section{Preliminaries }
An MDP \cite{puterman2004markov} can be represented as a tuple \( (S, A, R, P, \gamma ) \), where \textit{S} is a set of states, \textit{A} is a set of actions, \textit{P} is the transition probability function, \textit{R} is a reward function, and \(\gamma \) is a discount factor.
 The value of policy \(\pi\) at  state \(s\), \(V^\pi(s)\), is the  cumulative (discounted) expected reward, which is given by:
\({V^\pi }(s) = E\left[ {\sum\limits_{i =0}^\infty  {{{\gamma }^i}} R\left( {{s_{i}}, \pi ({s_{i}}),s_{i+1}} \right)} \mid s_0=s,\pi \right]\), where the expectation is over the sequence of states \(s_{i+1} \sim P(S|s_{i},\pi(s_{i}))\)  for all \(i\ge0\).
Using Bellman's equation, the value of the optimal policy or the optimal value, \(V^*(s)\), can be written as \(V^{*}(s) = \max_a \sum_{s'} P(s'|s_, a)) [R(s, a, s') + \gamma V^*(s')]\).

In many situations, the transition function \(P\) and/or the reward function \(R\) are initially unknown. Under such conditions, we often want a policy of an algorithm at time \(t\), \({\cal A}_t\), to yield a value \(V^{{\cal A}_t}(s_t)\)  that is close to the optimal value \(V^*(s_t)\) after some exploration. Here, \(s_{t}\) denotes the current state at time \(t\). More precisely, we may want the following: for all \(\epsilon>0\) and for all \(\delta=(0,1)\), \({{V^{{\cal A}_t}}(s_{t}) \geq V^* }(s_{t}) - \epsilon\), with probability at least \(1-\delta\) when \(t \ge \tau\), where \(\tau\) is the exploration time. The algorithm with a policy\({\cal A}_t\) is said to be ``probably approximately correct''\ for MDPs (PAC-MDP) \cite{strehl2007probably} if this condition holds with \(\tau\) being at most   polynomial in the relevant quantities of MDPs. The notion of PAC-MDP has a strong theoretical basis and is widely applicable, avoiding the need for additional assumptions, such as reachability in state space \cite{jaksch2010near}, access to a reset action \cite{fiechter1994efficient}, and access to a parallel sampling oracle \cite{kearns1999finite}.

However, the PAC-MDP approach often results in an algorithm over-exploring the state space, causing a low reward per unit time for a long period of time. Accordingly, past studies  that proposed PAC-MDP algorithms have rarely presented a corresponding experimental result, or have done so by tuning the free parameters, which renders the relevant algorithm no longer PAC-MDP \cite{strehl2006incremental,kolter2009near,sorg2010variance}. This problem was noted in  \cite{kolter2009near,brunskill2012bayes,kawaguchi2013greedy}. Furthermore, in many problems, it may not even be possible to guarantee \(V^{{\cal A}_t}\) close to \(V^*\) within the agent's lifetime. \citeauthor{li2012sample} \shortcite{li2012sample} noted that, despite the strong theoretical basis of the PAC-MDP approach, heuristic-based methods remain popular in practice. This would appear  to be a result of the above issues. In summary, there seems to be a dissonance between a strong theoretical approach and practical needs.

\section{Bounded Optimal Learning}
The practical limitations of the PAC-MDP approach lie in their focus on correctness without accommodating the time constraints that occur naturally in practice. To overcome the li\index{li }mitation, we first define the notion of \textit{reachability in model learning}, and then relax the  PAC-MDP objective based on it. For brevity, we focus on the transition model.

\subsection{Reachability in Model Learning}
For each state-action pair \((s, a)\), let \(M_{(s,a)}\) be a set of all transition models and \(\widehat P_{t}(\cdot|s, a)\in M_{(s,a)}\) be  the current model at time \(t\) (i.e., \(\widehat P_{t}(\cdot|s, a):S \rightarrow[0,\infty)\)). Define \(S'_{(s,a)}\) to be a set of possible future samples as \(S'_{(s,a)}=\{s'|P(s'|s,a)>0\}\). Let \(f_{(s,a)}:M_{(s,a)} \times S'_{(s,a)} \rightarrow M_{(s,a)} \) represent the model update rule; \(f_{(s,a)}\)  maps a model (in \(M_{(s,a)}\)) and a new sample (in \(S'_{(s,a)}\)) to a corresponding new model (in \(M_{(s,a)}\)). We can then write \(\mathcal{L}=(M,f)\) to represent a learning method of an algorithm, where \(M=\cup_{(s,a)\in(S,A)}M_{(s,a)}\) and \(f=\{f_{(s,a)}\}_{(s,a)\in (S,A)}\).

The set of \(h\)-reachable models, \({\mathcal M}_{\mathcal{L}, t,h,(s,a)}\), is recursively defined as \(\mathcal{M}_{\mathcal{L}, t, h,( s, a)} =\left\{\widehat P' \in M_{( s, a)} |\widehat P'=f_{( s, a)} (\widehat P, s')\right.\) for some \(\widehat P\in {\cal M}_{\mathcal{L},t,h-1,(s,a)}\) and \(\left.s'\in S'_{(s,a)}\right\}\) with the boundary condition \({\cal M}_{t,0,(s,a)} = \{\widehat P_{t}(\cdot|s, a) \}\).

Intuitively, the set of \(h\)-reachable models, \(\mathcal{M}_{\mathcal{L},t,h,(s,a)} \subseteq \mathcal{M}_{(s,a)}\), contains the transition models that can be obtained if the agent updates the current model at time \(t\) using any combination of  \(h\) additional samples \(s'_{1},s'_2,\ldots ,s'_h \sim P(S|s,a)\).  Note that the set of \(h\)-reachable models is defined \textit{separately for each state-action pair}. For example, \(\mathcal{M}_{\mathcal{L},t,h,(s_1,a_1)}\) contains only those models that are reachable  using the \(h\) additional samples drawn from \(P(S|s_{1},a_{1})\).

We define the \(h\)-reachable optimal value \(V^{d*}_{\mathcal L,t,h}(s)\) with respect to a distance function \(d\) as
\begin{small}
	\begin{equation*}
	V^{d*}_{\mathcal{L},t,h}(s) = \max_a  \sum_{s'} \widehat P^{d*}_{\mathcal{L},t,h}(s'|s, a) [R(s, a, s') + \gamma V^{d*}_{\mathcal{L},t,h}(s') ],
	\end{equation*}
\end{small}
where
\[\widehat P^{d*}_{\mathcal{L},t,h}(\cdot|s, a) = \argmin_{\widehat P \in{\cal M}_{\mathcal{L},t,h,(s,a)}} d(\widehat P(\cdot  |s, a), P(\cdot|\allowbreak s, a)).\]

Intuitively, the \(h\)-reachable optimal value, \(V^{d*}_{\mathcal L,t,h}(s)\), is the optimal value estimated with the ``best'' model in  the set of \(h\)-reachable models (here, the term ``best'' is in terms of the distance function \(d(\cdot,\cdot)\)).

\subsection{PAC in Reachable MDP}
Using the concept of reachability in model learning, we define the notion of ``probably approximately correct'' in an \(h\)-reachable MDP (PAC-RMDP(\(h\))). Let \(\mathcal P(x_{1},x_{2}, \ldots,x_{n})\) be a polynomial in \(x_{1},x_2, \ldots,x_n\) and \(|\text{MDP}|\) be the complexity of an MDP \cite{li2012sample}.

\begin{defn}
(PAC-RMDP(\(h\))) An algorithm with a policy \({\cal A}_t\) and a learning method \(\mathcal{L}$ is PAC-RMDP$(h)\)  with respect to a distance function \(d\) if for all \(\epsilon>0\) and for all \(\delta=(0,1)\),
\begin{enumerate}[label={\arabic*)}]
  \item there exists \small $\tau = O(\mathcal P(1/\epsilon,1/\delta,1/(1-\gamma),|\text{MDP}|,h))$ \normalsize such that
for all \(t\ge \tau\),
\[
V^{\mathcal A_t}(s_{t}) \geq V_{\mathcal{L},t,h}^{d*}(s_{t}) - \epsilon
\]
 with probability at least \(1 - \delta\), \textit{and}
  \item there exists \small $h^*(\epsilon, \delta)=O(\mathcal P(1/\epsilon,1/\delta,1/(1-\gamma),|\text{MDP}|))$ \normalsize such that
for all \(t \ge 0\),
\[
 |V^*(s_t) - V_{\mathcal L,t,h^*(\epsilon, \delta)}^{{d*}} (s_t)| \le \epsilon.
\]
with probability at least \(1 - \delta\).
\end{enumerate}
\end{defn}

The first condition ensures that the agent efficiently learns the \(h\)-reachable models. The second condition guarantees that the learning method \(\mathcal L\) and the distance function \(d\) are not arbitrarily poor.

In the following, we relate PAC-RMDP(\(h\)) to PAC-MDP and near-Bayes optimality. The proofs are given in the appendix at the end of this paper.

\begin{proposition}
(PAC-MDP) If an algorithm is PAC-RMDP(\small $h^*(\epsilon, \delta)$\normalsize), then it is PAC-MDP, where \small \(h^*(\epsilon, \delta)\) \normalsize is given in Definition 1.
\end{proposition}

\begin{proposition}
(Near-Bayes optimality) Consider model-based Bayesian reinforcement learning \cite{strens2000bayesian}. Let \(H\) be a planning horizon in the belief space \(b\). Assume that the Bayesian optimal value function, \(V_{b,H}^{*}\), converges to the \(H\)-reachable optimal function such that, for all \(\epsilon>0\), \(|V_{\mathcal{L},t,H}^{d*} (s_t) - V_{b,H}^{*}(s_t,b_t)| \le \epsilon\) for all but polynomial time steps. Then, a PAC-RMDP(\(H\)) algorithm with a policy \(\mathcal A_t\) obtains an expected cumulative reward \(V^{{\cal A}_t}(s_{t})\ge V_{b,H}^{*}(s_t,b_t) - 2\epsilon\)  for all but polynomial time steps with probability at least \(1 -\delta\).
\end{proposition}

\noindent Note that \(V^{{\mathcal{A}}_t}(s_{t})\) is the \textit{actual} expected cumulative reward with the expectation over the true dynamics \(P\), whereas \(V_{b,H}^{*}(s_t,b_t)\) is the \textit{believed} expected cumulative reward with the expectation over the current belief \(b_t\) and its belief evolution. In addition, whereas the PAC-RMDP(\(H\)) condition guarantees convergence to an \(H\)-reachable optimal value function, Bayesian optimality does  \textit{not}\footnote{A  Bayesian estimation with random samples converges to the true value under certain assumptions. However, for exploration, the selection of actions can cause the Bayesian optimal agent to ignore some state-action pairs, removing the guarantee of  the  convergence. This effect was well illustrated by \citeauthor{li2009unifying} (\citeyear{li2009unifying}, Example 9).}. In this sense, Proposition 2 suggests that the theoretical guarantee of PAC-RMDP(\(H\)) would be stronger than that of near-Bayes optimality with an \(H\) step lookahead.

Summarizing the above, PAC-RMDP(\(h^*(\epsilon,\delta)\)) implies PAC-MDP, and PAC-RMDP(\(H\)) is related to near-Bayes optimality.  Moreover, as \(h\) decreases in the range \((0, h^*)\) or \((0, H)\), the theoretical guarantee of PAC-RMDP(\(h\)) becomes weaker than previous theoretical objectives. This accommodates the practical need to improve the trade-off between the theoretical guarantee (i.e., optimal behavior after a long period of exploration) and practical performance (i.e., satisfactory behavior after a reasonable period of exploration) via the concept of reachability. We discuss the relationship to bounded rationality \cite{simon1982models} and bounded optimality \cite{russell1995provably} as well as the corresponding notions of regret and average loss in the appendix.

\section{Discrete Domain}
To illustrate the proposed concept, we first consider a simple case involving finite state and action spaces with an unknown transition function \(P\). Without loss of generality, we assume that the reward function \(R\) is known.

\subsection{Algorithm}
Let \(\tilde V^{\cal A}(s)\) be the internal value function used by the algorithm to choose an action. Let \( V^{\cal A}(s)\) be the actual value function according to true dynamics $P$. To derive the algorithm, we use the principle of optimism in the face of uncertainty, such that \(\tilde V^{\cal A}(s) \ge V^{d*}_{\mathcal{L},t,h}(s)\) for all \(s\in S\). This can be achieved using the following internal value function:
\fontsize{9pt}{9pt}
\begin{equation}
\tilde V^{\cal A} (s) = \hspace{-12pt} \max_{\substack{a, \\ \tilde P \in {\cal M}_{\mathcal{L},t,h,(s,a)}}} \hspace{-3pt} \sum_{s' } \tilde P(s'|s, a) [R(s, a, s') + \gamma
\tilde V^{\cal A} (s')]
 \label{Algorithm1}
\end{equation}
The pseudocode is shown in Algorithm 1. In the following, we consider the special case in which we use the sample mean estimator (which determines   \(\mathcal{L}\)). That is, we use \(\widehat P_{t}(s'|s, a) = n_t(s, a, s')/n_t(s, a)\), where \(n_t(s, a)\) is the number of samples for the state-action pair \((s, a)\), and \(n_t(s, a, s')\) is the number of samples for the transition from \(s\) to \(s'\) given an action \(a\). In this case, the maximum over the model in Equation (\ref{Algorithm1}) is achieved when all future \(h\) observations are transitions to the state with the best value. Thus,  \(\tilde V^{\cal A}\) can be computed by \(\tilde V^{\cal A}(s) = \max_{a} \sum_{s' \in S } \frac{n_t(s, a, s')}{n_t(s, a) + h} [ R(s, a, s') + \gamma \tilde V^{\cal A}(s') ] +\max_{s'} \frac{h}{n_t(s, a) + h}[ R(s, a, s') + \gamma \tilde V^{\cal A}(s')]\).

\begin{algorithm} [t!]
\caption{Discrete PAC-RMDP }
\label{Linear PAC-RMDP as2}
\begin{algorithmic}
    \small
    \REQUIRE \(h\ge0\)
    \vspace{+4pt}
    \FOR{time step \(t = 1, 2, 3, \ldots\)}
       \STATE Action: Take action based on \(\tilde V^{A}(s_t)\) in Equation (\ref{Algorithm1})        \STATE Observation: Save the sufficient statistics
       \STATE Estimate: Update the model \(\widehat P_{t,0}\)
    \ENDFOR
    \normalsize
    \vspace{-2pt}
\end{algorithmic}
\end{algorithm}

\subsection{Analysis}
We first show that Algorithm 1 is PAC-RMDP(\(h\)) for all  \(h\ge0\) (Theorem 1), maintains an anytime error bound and average loss bound (Corollary 1 and the following discussion),  and is related with previous algorithms (Remarks 1 and 2).  We then analyze its \textit{explicit exploration runtime} (Definition 3). We assume that Algorithm 1 is used with the sample mean estimator, which determines \(\mathcal L\). We fix the distance function as \(d(\widehat P(\cdot|s,a), P(\cdot|s,a))=\|\widehat P(\cdot|s,a) - P(\cdot|s,a) \|_1\). The proofs are given in the appendix.
\begin{theorem}
(PAC-RMDP) Let \(\mathcal A_t\) be a policy of Algorithm 1. Let $z = \max(h, \frac{\ln(2^{|S|}|S||A|/\delta)}{\epsilon(1 - \gamma)})$. Then, for all \(\epsilon>0\), for all $\delta=(0,1)$, and for all \(h\ge0\),
\begin{enumerate}[label={\arabic*)}]
\item for all but at most \( O\left(\frac{z|S||A|}{\epsilon^{2}(1 - \gamma)^{2}} \ln \frac{|S||A|}{\delta} \right)\) time steps, \(V^{\mathcal A_t}(s_{t}) \geq V_{\mathcal{L},t,h}^{d*}(s_{t}) - \epsilon\), with probability at least \(1 - \delta\), \textit{and}
\item there exist \small \(h^*(\epsilon, \delta)=O(\mathcal P(1/\epsilon,1/\delta,1/(1-\gamma),|\text{MDP}|))\) \normalsize  such that \( |V^*(s_t) - V_{\mathcal{L},t,h^*(\epsilon,\delta)}^{d*} (s_{t})| \le \epsilon\) with probability at least \(1 - \delta\).
\end{enumerate}
\end{theorem}

\begin{defn}
(Anytime error)  The anytime error \(\epsilon_{t,h} \in \mathbb{R}\) is the smallest value such that \(V^{\mathcal A_t}(s_{t}) \geq V_{\mathcal{L},t,h}^{d*}(s_{t}) - \epsilon_{t,h}\).
\end{defn}

\begin{corollary}
(Anytime error bound) With probability at least \(1 - \delta\), if \(h\le\frac{\ln(2^{|S|}|S||A|/\delta)}{\epsilon(1 - \gamma)}\), \(\epsilon_{t,h} = O\left( \sqrt[3]{ \frac{|S||A|}{t(1-\gamma)^3} \ln \frac{|S||A|}{\delta} \ln \frac{2^{|S|}|S||A|}{\delta} } \right);\) otherwise, \(\epsilon_{t,h} = O \left( \sqrt{\frac{h|S||A|}{t(1-\gamma)^2} \ln \frac{|S||A|}{\delta}} \right)\).

The anytime \(T\)-step average loss is equal to \(\frac{1}{T}\sum_{t=1}^{T} (1-\allowbreak\gamma^{T+1-t})\epsilon_{t,h}\). Moreover, in this simple problem, we can relate Algorithm 1 to a particular PAC-MDP\ algorithm and a near-Bayes optimal algorithm.
\end{corollary}

\begin{remark}
(Relation to MBIE) Let \small \(m = O(\frac{|S|}{\epsilon^2(1 - \gamma)^4} + \frac{1}{\epsilon^2(1 - \gamma)^4} \ln \frac{|S||A|}{\epsilon(1 - \gamma)\delta})\). \normalsize Let \small \(h^*(s, a) = \frac{n(s, a)z(s, a)}{1 - z(s, a)}\), \normalsize where  \small $z(s, a) = 2 \sqrt{2[\ln(2^{|S|}  - 2) -   \ln(\delta/(2|S||A|m)) ]/n(s, a)}$. \normalsize Then, Algorithm 1 with the input parameter \(h = h^*(s, a)\) behaves identically to a PAC-MDP algorithm, Model Based  Interval Estimation (MBIE) \cite{strehl2008analysis}, the sample complexity of which is \(O(\frac{|S||A|}{\epsilon^3(1 -\gamma)^6}(|S| \allowbreak + \ln\frac{|S||A|}{\epsilon(1 - \gamma)\delta})\ln\frac{1}{\delta}\ln\frac{1}{\epsilon(1 - \gamma)}))\).
\end{remark}

\begin{remark}
(Relation to BOLT) Let \(h = H\), where \(H\) is a planning horizon in the belief space $b$. Assume that Algorithm 1 is used with an independent Dirichlet  model for each $(s,a)$, which determines $\mathcal L$. Then, Algorithm 1 behaves identically to a near-Bayes optimal algorithm, Bayesian Optimistic Local Transitions (BOLT) \cite{araya2012near}, the sample complexity of which is  \(O(\frac{H^2|S||A|}{\epsilon^2(1 - \gamma)^2}\ln\frac{|S||A|}{\delta})\).
\end{remark}

As expected, the sample complexity for  PAC-RMDP(\(h\)) (Theorem 1) is smaller than that for PAC-MDP (Remark 1) (at least when \(h\le |S|(1 -\gamma)^{-3}\)), but larger than that for near-Bayes optimality (Remark 2) (at least when \(h \ge H\)). Note that BOLT is not necessarily PAC-RMDP(\(h\)), because misleading priors can violate both conditions in Definition 1.

\subsubsection{Further Discussion}
An important  observation is that, when \(h \le \frac{|S|}{\epsilon(1 - \gamma)}\ln \frac{|S||A|}{\delta} \), the sample complexity of Algorithm 1 is dominated by the number of samples required to refine the model, rather than the explicit exploration of unknown aspects of the world. Recall that the internal value function \(\tilde V^{\cal A}\) is designed to force the agent to explore, whereas the use of the currently estimated value function  \(V^{d*}_{\mathcal{L},t,0}(s)\) results in exploitation. The difference between \(\tilde V^{\cal A}\) and \(V^*_{\mathcal{L},t,0}(s)\) decreases at a rate of \(O(h/n_t(s, a))\), whereas the error between \(V^{\cal A}\) and \(V^{d*}_{\mathcal{L},t,0}(s)\) decreases at a rate of \(O(1/\sqrt{n_t(s, a))}\). Thus, Algorithm 1 would stop the explicit exploration much sooner (when \(\tilde V^{\cal A}\) and \(V^{d*}_{\mathcal{L},t,0}(s)\) become close), and begin exploiting the model, while still refining it, so that \(V^{d*}_{\mathcal{L},t,0}(s)\) tends  to \(V^{\cal A}\). In contrast, PAC-MDP algorithms are forced to explore until the error between \( V^{\cal A}\) and \(V^*\) becomes sufficiently small, where the error decreases at a rate of \(O(1/\sqrt{n_t(s, a))}\). This  provides some intuition to explain why a PAC-RMDP(\(h\)) algorithm with small \(h\) may avoid over-exploration, and yet, in some cases, learn the true dynamics to a reasonable degree, as shown in the experimental examples.

In the following, we formalize the above discussion.
\\
\begin{defn}
(Explicit exploration runtime) The \textit{explicit exploration runtime} is the smallest integer $\tau$ such that for all $t\ge\tau$, \(|\tilde V^{\mathcal A_t}(s_t) -  V^{d*}_{\mathcal{L},t,0}(s_{t})| \le \epsilon\).
\end{defn}

\begin{corollary}
(Explicit exploration bound)  With probability at least \(1 - \delta\), the explicit exploration runtime of Algorithm 1 is \( O(\frac{h|S||A|}{\epsilon(1 - \gamma)\Pr[A_K]} \ln \frac{|S||A|}{\delta} )=O(\frac{h|S||A|}{\epsilon^{2}(1 - \gamma)^{2}} \ln \frac{|S||A|}{\delta} )\), where \(A_{K}\) is the escape event defined in the proof of Theorem 1.
\end{corollary}

If we assume  \(\Pr[A_K]\) to stay larger than a fixed constant, or to be very small (\(\le \frac{\epsilon(1 - \gamma)}{3R_{max}}\)) (so that \(\Pr[A_K]\) does not appear in Corollary 2 as shown in the corresponding case analysis for Theorem 1),  the explicit exploration runtime can be reduced to \(O(\frac{h|S||A|}{\epsilon(1 - \gamma)} \ln \frac{|S||A|}{\delta})\).  Intuitively, this happens when the given MDP does not have low yet not-too low probability and high-consequence   transition that is initially unknown. Naturally,  such a MDP  is  difficult to learn, as reflected in Corollary 2.

\subsection{Experimental Example}
We compare the proposed algorithm with MBIE \cite{strehl2008analysis}, variance-based exploration (VBE) \cite{sorg2010variance}, Bayesian Exploration Bonus (BEB) \cite{kolter2009near}, and BOLT \cite{araya2012near}. These algorithms were designed to be PAC-MDP or near-Bayes optimal, but have been used with parameter settings that render them neither PAC-MDP nor near-Bayes optimal. In contrast to the experiments in previous research, we present results with $\epsilon$  set to several theoretically meaningful values\footnote{MBIE is PAC-MDP with the parameters \(\delta\) and \(\epsilon\). VBE is PAC-MDP in the assumed (prior) input distribution with the parameter \(\delta\). BEB and BOLT are near-Bayes optimal algorithms whose parameters \(\beta\) and \(\eta\) are fully specified by their analyses, namely  \(\beta = 2H^2\) and \(\eta = H\). Following \citeauthor{araya2012near} \shortcite{araya2012near}, we set \(\beta\) and \(\eta\) using the \(\epsilon'\)-approximated horizon \(H \approx \ceil{\log_{\gamma} (\epsilon'(1 - \gamma))} = 148\). We use the sample mean estimator for the PAC-MDP and PAC-RMDP(\(h\)) algorithms, and an independent Dirichlet model for the near-Bayes optimal algorithms.} as well as one theoretically non-meaningful value  to illustrate its property\footnote{\label{foot_exp1}We can interpolate their qualitative behaviors with values of $\epsilon$ other than those presented here. This is because the principle behind our results is that small values of $\epsilon$ causes over-exploration due to the focus on the near-optimality.  }. Because our algorithm is deterministic with no sampling and no assumptions on the input distribution, we do not compare it with algorithms that use sampling, or rely heavily on knowledge of the input distribution.

We consider a five-state chain problem \cite{strens2000bayesian}, which is a standard toy problem in the literature. In this problem, the optimal policy is to move toward the state farthest from the initial state, but the reward structure explicitly encourages an exploitation agent, or even an \(\epsilon\)-greedy agent, to remain in the initial state. We use a discount factor of \(\gamma = 0.95\) and a convergence criterion for the value iteration of \(\epsilon' = 0.01\).

Figure 1 shows the numerical results in terms of the average reward per time step (average over 1000 runs). As can be seen from this figure, the proposed algorithm worked better. MBIE and VBE work reasonably if we discard the theoretical guarantee. As the maximum reward is \(R_{max} = 1\), the upper bound on the value function is \(\sum_{i = 1}^\infty \gamma^i R_{max} = \frac{1}{1 - \gamma}R_{max} = 20\). Thus, \(\epsilon\)-closeness does not yield any useful information when \(\epsilon \ge 20\). A similar problem was noted by \citeauthor{kolter2009near} \shortcite{kolter2009near} and \citeauthor{araya2012near} \shortcite{araya2012near}.

In the appendix, we present the results for a problem with low-probability high-consequence   transitions, in which PAC-RMDP(8) produced the best result.

\begin{figure}
   \includegraphics[width=\columnwidth]{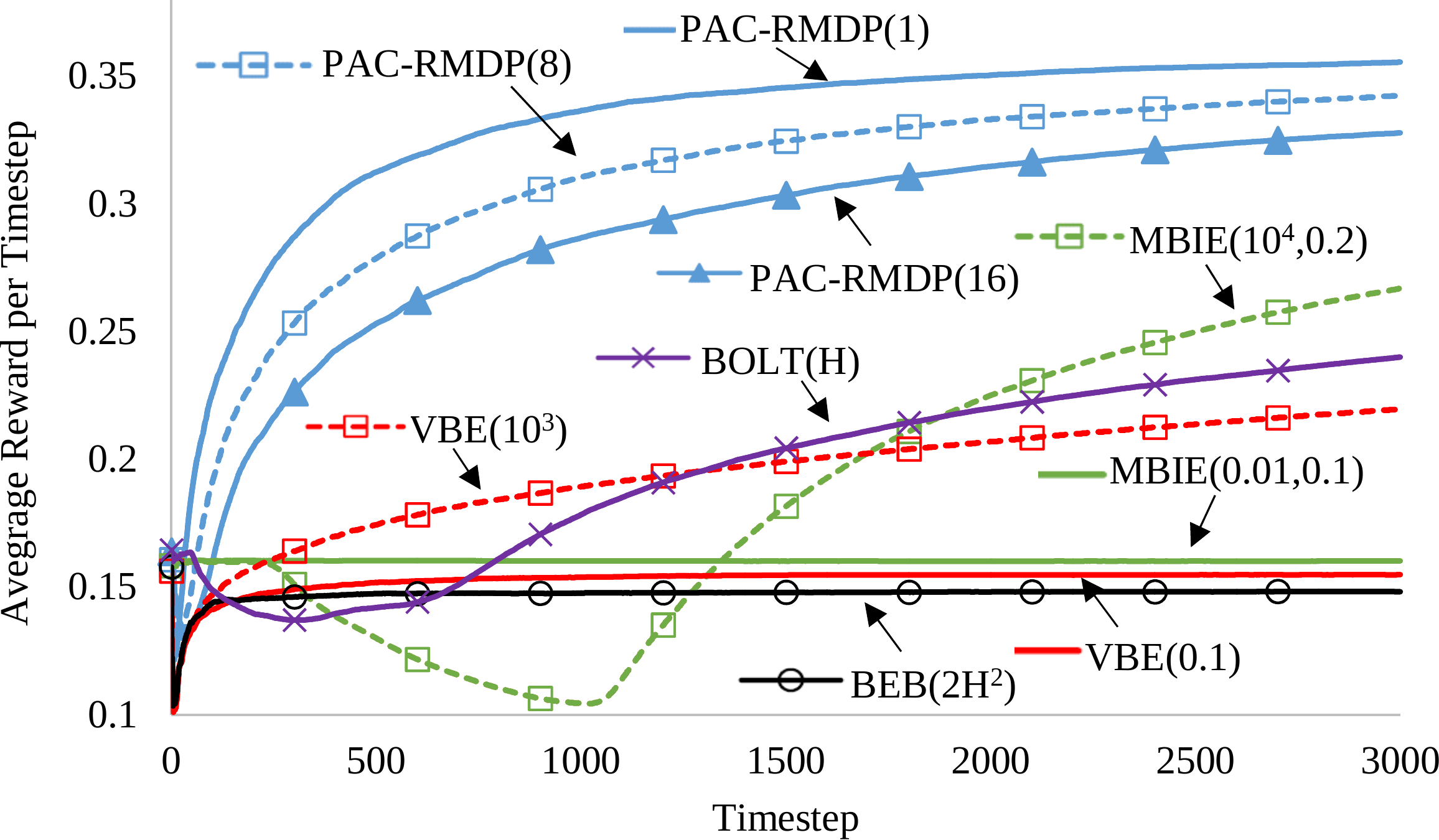}
   \caption{Average total reward per time step for the Chain Problem. The algorithm parameters are shown as PAC-RMDP(\(h\)), MBIE($\epsilon, \delta$), VBE($\delta$), BEB($\beta$), and BOLT($\eta$).}
\end{figure}

\section{Continuous Domain}
In this section, we consider the  problem of a continuous state space and discrete action space. The transition function is possibly nonlinear, but can be linearly parameterized as: \(s_{t + 1}^{(i)} = \theta _{(i)}^T{\Phi _{(i)}}({s_{t}}, {a_{t}}) + \zeta_{t}^{(i)},\) where the state  \(s_t \in S \subseteq \mathbb{R}^{n_S}\) is  represented by \(n_{S}\) state parameters (\(s^{(i)} \in \mathbb{R }\) with \(i \in \{1, \ldots, n_s\}\)), and \(a_t \in A\) is the action at time \(t\). We assume that the basis functions \(\Phi_{(i)}:S \times A \rightarrow\ \mathbb{R}^{n_{i}}\) are known, but the weights \(\theta \in \mathbb{R}^{n_{i}}\) are unknown. \(\zeta_{t}^{(i)} \in \mathbb{R}\) is the noise term and given by \(\zeta_{t}^{(i)} \sim \mathcal N (0,\sigma^{2}_{(i)})\). In other words, \(P(s^{(i)}_{t+1}|s_t,a_t)=\mathcal N ( \theta _{(i)}^T{\Phi _{(i)}}({s_{t}}, {a_{t}}),\sigma^{2}_{(i)})\). For brevity, we focus on unknown transition dynamics, but our method is directly applicable to unknown reward functions if the reward is represented in the above form. This problem is a slightly generalized version of those considered by \citeauthor{abbeel2005exploration} \shortcite{abbeel2005exploration}, \citeauthor{strehl2008online} \shortcite{strehl2008online}, and \citeauthor{li2011knows} \shortcite{li2011knows}.

\subsection{Algorithm}
We first define the variables used in our algorithm, and then explain how the algorithm works. Let \(\hat \theta_{(i)}\) be the vector of the model parameters for the \(i^{th}\) state component. Let \(X_{t,i} \in \mathbb{R}^{t \times n_{i}}\) consist of \(t\) input vectors \(\Phi^T_{(i)} (s, {a}) \in \mathbb{R}^{ 1\times n_{i}}\) at time \(t\). We then denote the eigenvalue decomposition of the input matrix as
\(X_{t,i}^T{X_{t,i}} = U_{t,i}D_{t,i}(\lambda_{(1)}, \ldots , \lambda_{(n)})U_{t,i}^T,\) where \(D_{t,i}(\lambda_{(1)}, \ldots, \lambda_{(n)}) \in \mathbb{R}^{n_{i}\times n_{i}}\) represents a diagonal matrix. For simplicity of notation,  we arrange the eigenvectors and eigenvalues such that the diagonal elements of \(D_{t,i}(\lambda_{(1)}, \dots, \lambda_{(n)})\) are \(\lambda_{(1)}, \ldots, \lambda_{(j)} \geq 1\) and \(\lambda_{(j + 1)}, \ldots, \lambda_{(n)} < 1 \) for some \(0\le j \le n\). We now define the main variables used in our algorithm: \(z_{t,i} := (X_{t,i}^T{X_{t,i}})^{ - 1}, g_{t,i} := U_{t,i}D_{t,i}(\frac{1}{\lambda_{(1)}}, \ldots , \frac{1}{\lambda_{(j)}}, 0, \ldots , 0)U_{t,i}^T\),
and \(w_{t,i} := U_{t,i}{D_{t,i}}(0, \ldots , 0, {1_{(j + 1)}}, \ldots , {1_{(n)}}){U_{t,i}^T}\). Let \(\Delta^{(i)} \ge\sup_{s, a} |(\theta_{(i)} -\hat \theta_{(i)})^T \Phi_{(i)}(s, a)|\) be the upper bound on the model error. Define \(\varsigma(M)=\sqrt{2\ln(\pi^2M^2n_sh/(6\delta))}\)  where \(M\) is the number of calls for \(\mathbf I_h\) (i.e., the number of computing \(\tilde r\) in Algorithm 2).

With the above variables, we define the \(h\)-reachable model interval \(I_h\) as
\begin{dmath*}
I_h(\Phi_{(i)}(s, a), X_{t,i})/[h(\Delta^{(i)}+\varsigma(M)\sigma_{(i)})] = \\ |\Phi_{(i)}^T(s, a) g_{t,i} \Phi_{(i)}(s, a)| + \| \Phi_{(i)}^T(s, a)z_{t,i}\| \| w_{t,i}\Phi_{(i)}(s, a) \|.
\end{dmath*}

The \(h\)-reachable model interval is a function that maps a new state-action pair considered in the planning phase, \(\Phi_{(i)}(s, a)\), and the agent's experience, \(X_{t,i}\), to the upper bound of the error in the model prediction. We define the column vector consisting of  \(n_S\) \(h\)-reachable intervals as \(\mathbf I_h(s, a, X_t) = [I_h(\Phi_{(1)}(s, a), X_{t,1}), \ldots, I_h(\Phi_{(n_{S})}(s, a), X_{t,n_S})]^T\).

We also leverage the continuity of  the internal value function $\tilde V$ to avoid an expensive computation (to translate the error in the model to the error  in value).

\begin{assumption}
(Continuity) There exists \(L\in\mathbb{R}\) such that, for all \(s,s' \in S\), \(|\tilde V^*(s) - \tilde V^*(s')| \le L\|s - s'\|\).
\end{assumption}

We set the degree of optimism for a state-action pair to be proportional to the uncertainty of the associated model. Using the \(h\)-reachable model interval, this can be achieved by simply adding a reward bonus that is proportional to the interval. The pseudocode for this is shown in Algorithm 2.

\subsection{Analysis}
Following previous work \cite{strehl2008online,li2011knows}, we assume access to an exact planning algorithm. This assumption would be relaxed by using a planning method that provides an error bound. We assume that Algorithm 2 is used with least-squares estimation, which determines \(\mathcal L\). We fix the distance function as \(d(\widehat P(\cdot|s,a), P(\cdot|s,a))=|E_{s' \sim \widehat P(\cdot|s,a)}[s']-E_{s' \sim P(\cdot|s,a)}[s']|\) (since the unknown aspect is the mean, this choice makes sense). In the following, we use \(\bar n\) to represent the average value of \(\{n_{(1)}, \ldots, n_{(n_{S})}\}\).
The proofs are given in the appendix.
\setcounter{lemma}{2}
\begin{lemma}
(Sample complexity of PAC-MDP) For our problem setting, the PAC-MDP algorithm proposed by \citeauthor{strehl2008online} \shortcite{strehl2008online} and \citeauthor{li2011knows} \shortcite{li2011knows} has sample complexity \(\tilde O\left(\frac{{n_S^2{}\bar n^2}}{{{\epsilon ^5}{{(1 - \gamma )}^{10}}}} \right)\).
\end{lemma}

\begin{algorithm} [t!]
\caption{Linear PAC-RMDP}
\label{Linear PAC-RMDP}
\begin{algorithmic}
    \small
    \REQUIRE \(h,\delta\) \: Optional: \(\Delta^{(i)},L\)
    \vspace{2pt}
    \STATE Initialize:  \(\hat \theta\), \(\Delta_{}^{(i)}\), and \(L\)
    \FOR{time step \(t = 1, 2, 3, ...\ldots\)}

      \STATE Action: take an action based on
      \STATE \hspace{15pt} \(\hat p(s'|s, a)\leftarrow{\cal N} (\hat\theta^T\Phi(s, a), \sigma^{2} I)\)
      \STATE \hspace{15pt} \(\tilde r(s, a, s')\leftarrow R(s, a, s') + L\|\mathbf I_h(s, a, X_{t - 1})\|\)

     \STATE Observation: Save the input-output pair \((s_{t + 1}, \Phi_{t}(s_{t}, a_{t}))\)
      \STATE Estimate: \hspace{-3pt} Estimate  \(\hat \theta_{(i)}\), \(\Delta^{(i)}\)  \hspace{-2pt}(if not given), \hspace{-2pt} and \(L\) \hspace{-3pt} (if not given)
    \ENDFOR
    \normalsize
\vspace{-2pt}
\end{algorithmic}
\end{algorithm}

\begin{theorem}
(PAC-RMDP) Let \(\mathcal A_t\) be the policy of Algorithm 2. Let \(z = \max(h^2\ln \frac{m^2n_sh}{\delta}, \frac{L^2n_S \bar n\ln^2 m}{\epsilon^3}\ln \frac{n_S}{\delta})\).  Then, for all \(\epsilon>0\), for all $\delta=(0,1)$, and for all \(h\ge0\),
\begin{enumerate}[label={\arabic*)}]
\item for all but at most \(m'=O\left( {\frac{zL^2 n_S\bar n\ln^2 m}{\epsilon^3 (1 - \gamma )^2}}\ln^2\frac{n_S}{\delta} \right)\) time steps (with \(m\le m'\)), \(V^{\mathcal A_t}(s_{t}) \geq V_{\mathcal{L},t,h}^{d*}(s_{t}) - \epsilon\), with probability at least \(1 - \delta\), \textit{and}
\item there exists \small \(h^*(\epsilon, \delta)=O(\mathcal P(1/\epsilon,1/\delta,1/(1-\gamma),|\text{MDP}|))\) \normalsize such that \( |V^*(s_t) - V_{\mathcal{L},t,h^*(\epsilon,\delta)}^{d*}(s_{t})| \le \epsilon\) with probability at least \(1 - \delta\).
\end{enumerate}
\end{theorem}

\begin{corollary}
(Anytime error bound) With probability at least \(1-\delta\), if \(h^2\ln \frac{m^2n_sh}{\delta}\le \frac{L^2n_S \bar n\ln^2 m}{\epsilon^3}\ln \frac{n_S}{\delta}\), \(\epsilon_{t,h} = O\left( \sqrt[5]{ {\frac{L^{4} n_S^{2}\bar n^{2}\ln^{2} m}{t(1 - \gamma )}}\ln^3\frac{n_S}{\delta} } \right)\); otherwise, \(\epsilon_{t,h} = O\left( {\frac{h^2L ^2n_S^{}\bar n^{}\ln^2 m}{t(1 - \gamma )}}\ln^2\frac{n_S}{\delta} \right)\).
\end{corollary}
The anytime \(T\)-step average loss is equal to \(\frac{1}{T}\sum_{t=1}^{T} (1-\allowbreak\gamma^{T+1-t})\epsilon_{t,h}\).
\begin{corollary}
(Explicit exploration runtime) With probability at least \(1-\delta\), the explicit exploration runtime of Algorithm 2 is \( O\left({\frac{h^2L^2 n_S^{}\bar n\ln m}{\epsilon^2\Pr[A_k]  }}\ln^2\frac{n_S}{\delta} \ln \frac{m^2n_sh}{\delta}\right)=O\left({\frac{h^2L^2 n_S^{}\bar n\ln m}{\epsilon^{3}(1-\gamma)  }}\ln^2\frac{n_S}{\delta} \ln \frac{m^2n_sh}{\delta} \right)\), where \(A_{K}\) is the escape event defined in the proof of Theorem 2.
\end{corollary}

\subsection{Experimental Examples}
We consider two examples: the mountain car problem \cite{sutton1998reinforcement}, which is a standard toy problem in the literature, and the HIV problem \cite{ernst2006clinical}, which originates from a real-world problem. For both examples, we compare the proposed algorithm with a directly related PAC-MDP algorithm \cite{strehl2008online,li2011knows}. For the PAC-MDP algorithm, we present the results with $\epsilon$  set to several theoretically meaningful values and one  theoretically non-meaningful value to illustrate its property\footnote{See footnote \ref{foot_exp1} on the consideration of different values of \(\epsilon\).}. We used \(\delta = 0.9\) for the PAC-MDP and PAC-RMDP algorithms\footnote{We considered \(\delta\) = \([0.5, 0.8, 0.9, 0.95]\), but there was no change in any qualitative behavior of interest in our discussion.}. The \(\epsilon\)-greedy algorithm is executed with \(\epsilon = 0.1\). In the planning phase, \(L\) is estimated as \(L \leftarrow \max_{s,s' \in \Omega } |\tilde V^{\cal A}(s) - \tilde V^{\cal A}(s')| / \|s - s'\|\), where \(\Omega\) is the set of states that are visited in the planning phase (i.e., fitted value iteration and a greedy roll-out method). For both problems, more detailed descriptions of the experimental settings are available in the appendix.

\subsubsection{Mountain Car}
In the mountain car problem, the reward is negative everywhere except at the goal. To reach the goal, the agent must first travel far away, and must explore the world to learn this mechanism. Each episode consists of 2000 steps, and we conduct simulations for 100 episodes.

The numerical  results are shown in Figure 2. As in the discrete case, we can see that the PAC-RMDP(\(h\)) algorithm worked well. The best performance, in terms of the total reward, was achieved by PAC-RMDP(10). Since this problem required a number of consecutive explorations, the random exploration employed by the \(\epsilon\)-greedy algorithm did not allow the agent to reach the goal. As a result of exploration and the randomness in the environment, the PAC-MDP algorithm reached the goal several times, but kept exploring the environment to ensure near-optimality. From Figure 2, we can see that the PAC-MDP algorithm quickly converges to good behavior if we discard the theoretical guarantee (the difference between the values in the optimal value function had an upper bound of 120, and the total reward had an upper bound of 2000. Hence, \(\epsilon > 2000\) does not yield a useful theoretical guarantee).

\subsubsection{Simulated HIV Treatment}
This problem is described by a set of six ordinary differential equations \cite{ernst2006clinical}. An action corresponds to whether the agent administers two treatments (RTIs and PIs) to patients (thus, there are four actions). Two types of exploration are required: one to learn the effect of using treatments on viruses, and another to learn the effect of not using treatments on immune systems. Learning the former is necessary to reduce the population of viruses, but the latter is required to prevent the overuse of treatments, which weakens the immune system.  Each episode consists of 1000 steps (i.e., days), and we conduct simulations for 30 episodes.

As shown in Figure 3, the PAC-MDP algorithm worked reasonably well with \(\epsilon = 30^{10}\). However, the best total reward did not exceed \(30^{10}\), and so the PAC-MDP guarantee with \(\epsilon = 30^{10}\) does not seem to be useful. The \(\epsilon\)-greedy algorithm did not work well, as this example required sequential exploration at certain periods to learn the effects of treatments.

\section{Conclusion}
\begin{figure}[htp]
   \includegraphics[width=\columnwidth]{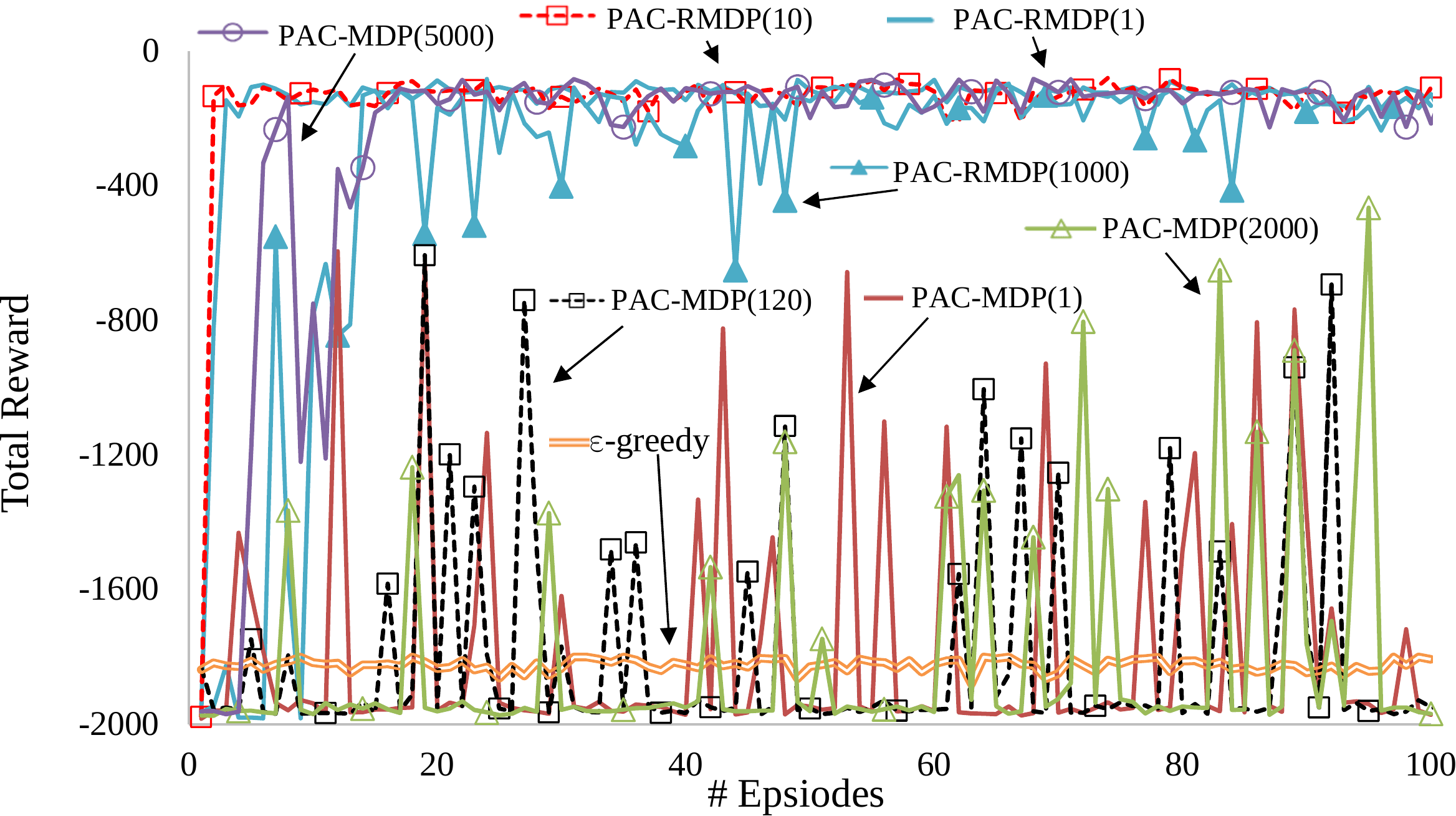}
   \caption{Total reward per episode for the mountain car problem with  PAC-RMDP(\(h\)) and PAC-MDP($\epsilon$).}
\end{figure}
\begin{figure}[htp]
   \includegraphics[width=\columnwidth]{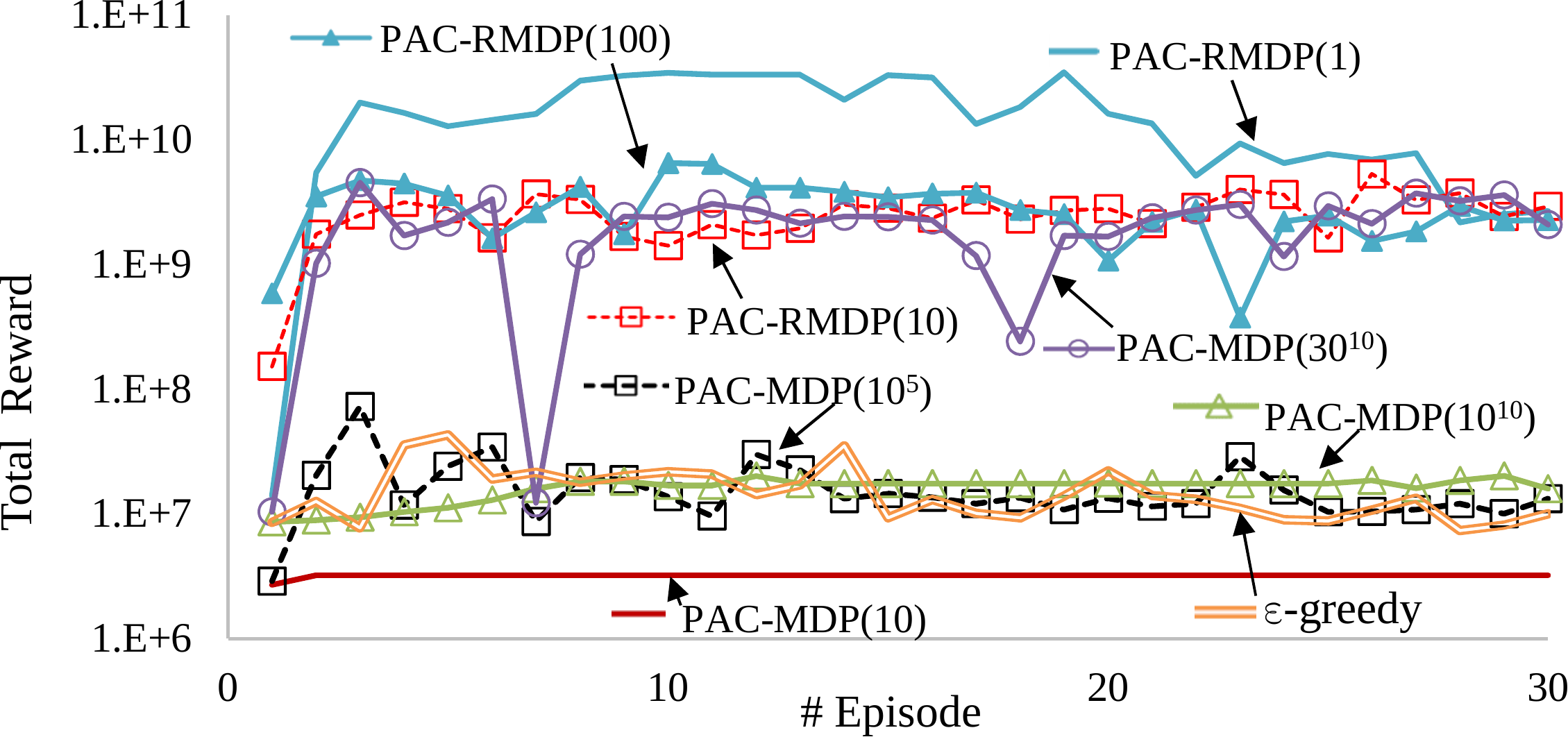}
   \caption{Total reward per episode for the HIV problem with PAC-RMDP(\(h\)) and PAC-MDP($\epsilon$).}
\end{figure}
In  this paper, we have proposed the PAC-RMDP framework to bridge the gap between theoretical objectives and practical needs. Although the PAC-RMDP(\(h\)) algorithms worked well in our experimental examples with small \(h\), it is possible  to devise a problem in which the PAC-RMDP algorithm should be used with large \(h\). In extreme cases, the algorithm would reduce to PAC-MDP. Thus, the adjustable theoretical guarantee of PAC-RMDP(\(h\)) via the concept of reachability seems to be a reasonable objective.

Whereas the development of algorithms with traditional objectives (PAC-MDP or regret bounds) requires the consideration of confidence intervals, PAC-RMDP(\(h\)) concerns a set of  \(h\)-reachable models. For a flexible model, the derivation of the confidence interval would be a difficult task, but a set of \(h\)-reachable models can simply be  computed (or approximated) via lookahead using the model update rule. Thus, future work includes the derivation of a PAC-RMDP algorithm with a more flexible and/or structured model.

\section*{Acknowledgment}
The author would like to thank Prof. Michael Littman, Prof. Leslie Kaelbling and Prof. Tom\'{a}s Lozano-P\'{e}rez for their thoughtful comments and suggestions. We gratefully acknowledge support from NSF grant 1420927, from ONR grant N00014-14-1-0486, and from ARO grant W911NF1410433. Any opinions, findings, and conclusions or recommendations expressed in this material are those of the authors and do not necessarily reflect the views of our sponsors.

\onecolumn
\setcounter{theorem}{0}
\setcounter{equation}{0}
\setcounter{lemma}{0}
\setcounter{proposition}{0}
\setcounter{corollary}{0}
\setcounter{defn}{0}
\setcounter{remark}{0}
\setcounter{assumption}{0}
\setcounter{figure}{0}
{\vspace{10mm}
\LARGE\centering\bf Appendix A \par \normalsize
\vspace{10mm}
}

\section{A1. Proofs of Propositions 1 and 2}
In this section, we present the proofs of Propositions 1 and 2.

\begin{proposition}
	(PAC-MDP) PAC-RMDP(\small$h^*(\epsilon, \delta)$\normalsize) implies PAC-MDP, where \small \(h^*(\epsilon, \delta)\) \normalsize is given in Definition 1.
\end{proposition}

\begin{proof}
For any PAC-RMDP($h^*(\epsilon,\delta)$) algorithm, Definition 1 implies that \({{V^{\cal A}}(s_{t}) \geq V_{h^*}^* }(s_{t}) - \epsilon \ge V^*(s_{t}) - 2\epsilon\) with probability at least \(1 - 2\delta\) for all but polynomial time steps. This satisfies the condition of the PAC-MDP.
\end{proof}

\begin{proposition}
	(Near-Bayes optimality) Consider the model-based Bayesian reinforcement learning \cite{strens2000bayesian}. Let $H$ be a planning horizon in the belief space $b$. Assume that the Bayesian optimal value function, $V_{b,H}^{*}$, converges to the $H$-reachable optimal function such that, for all \(\epsilon>0\), $ |V_{\mathcal{L},t,H}^* (s_t) - V_{b,H}^{*}(s_t,b_t)| \le \epsilon$ for all but polynomial time steps. Then, a PAC-RMDP($H$) algorithm with a policy \(\mathcal A_t\) obtains an expected cumulative reward ${V^{{\cal A}_t}}(s_{t})\ge V_{b,H}^{*}(s_t,b_t) - 2\epsilon$  for all but polynomial time steps with probability at least \(1 -\delta\).
\end{proposition}

\begin{proof}
It directly follows Definition 1 and the assumption. For all but polynomial time steps, with probability at least \(1 - \delta\), \(V^{\cal A}(s_{t})  \ge V_{\mathcal{L},t,H}^* (s_{t}) - \epsilon \ge V_{b,H}^{*}(s_t,b_t) - 2\epsilon \).
\end{proof}

\section{A2. Relationship to Bounded Rationality and Bounded Optimality}
As the concept of PAC-RMDP considers the inherent limitations of a decision maker, it shares properties with the concepts of bounded rationality \cite{simon1982models} and bounded optimality \cite{russell1995provably}.

Bounded rationality and bounded optimality focus on limitations in the planning phase (e.g., computational resources). In contrast, PAC-RMDP considers limitations in the learning phase (e.g., the agent's lifetime). As in the case of bounded rationality, the performance guarantee of a PAC-RMDP($h$) algorithm can be arbitrary, depending on the choice of \(h\). On the contrary, bounded optimality solves the problem of arbitrariness, seemingly at the cost of applicability. It requires a strong notion of optimality, similar to instance optimality; roughly, we must find the \textit{optimal algorithm} given the available computational resources. Automated optimization over the set of algorithms is a difficult task. \citeauthor{zilberstein2008metareasoning} \shortcite{zilberstein2008metareasoning} claims that bounded optimality is difficult to achieve, resulting in very few successful examples, and is not, in practice, as promising as other bounded rational methods. However, in future research, it would be interesting to compare PAC-RMDP with a possible relaxation of PAC-MDP based on a concept similar to bounded optimality.

\section{A3. Corresponding Notions of Regret and Average Loss}
In the definition of PAC-RMDP($h$), our focus is on \textit{learning} useful models, enabling us to obtain high rewards in a short period of time. Instead,  one may wish to guarantee the worst  total reward \textit{in a given time horizon \(T\)}. There are several ways to achieve this goal. One solution is to
minimize the expected  \(T\)-step \textit{regret bound} \(r^{\cal A}(T)\), given by
\begin{equation}
r^{\cal A}(T)\ge{V^* }(s_{0}, T) -{V^{\cal A} }(s_{0}, T).
\end{equation}

In this case, \(V^{*}(s_{0}, T)= E\left[\sum_{i= 0}^T  {{{\gamma }^i}} R\left( {{s_i^*}, \pi^*({s_i}),s_{i+1}^*} \right)\right] \), where the sequence of states \(s_0^*,s_1^*, \ldots,s_T^*\)  with \(s_0^*=s_0\) is  generated when  the agent follows the optimal policy \(\pi^*\) from  \(s_0\), and \(V^{\cal A}(s_{0}, T)= E\left[\sum_{i= 0}^T  {{{\gamma }^i}} R\left( {{s_i}, {\cal A}_i({s_i}),s_{i+1}}\right)\right]\), where the sequence of states \(s_0,s_1, \ldots,s_T\) is generated when  the agent follows policy \(\mathcal{A}_i\). Since one mistake in the early stages may make it impossible to return to the optimal state sequence \(s_i^*\), all the regret approaches in the literature rely on some  reachability assumptions in the state space; for example, \citeauthor{jaksch2010near} \shortcite{jaksch2010near} assumed that every state was reachable from every other state within a certain (average) number of steps.

Another approach is to minimize the expected  \(T\)-step \textit{average loss bound}  \(r^{\cal A}(T)\), which obviates the need for any reachability assumptions in the state space:
\begin{equation}
\ell^{\mathcal A}(T) \ge \frac{1}{T} \sum_{t=0}^{T} \left[V^*(s_{t}, T) - V^{\cal A}(s_{t}, T)\right],
\end{equation}
where \(s_{t}\) is the state visited by algorithm \(\mathcal{A}\) at time \(t\). The value functions inside the sum are defined as \(V^{*}(s_{t}, T) = E\left[\sum_{i =0}^{T-t} \gamma^{i} R\left( s_{t+i}^*, \pi^*(s_{t+i}),s_{t+i+1}^* \right)\right] \) with \(s_t^*=s_t\) and \(V^{\mathcal{A}}(s_{0}, T)= E\left[ \sum_{i= 0}^{T-t}  \gamma^{i} R\left( s_{t+i}^*, \mathcal{A}_t(s_{t+i}),s_{t+i+1} \right)\right] \). By averaging  the \(T\)-step regrets (i.e., losses) of the \(T\) initial states \(s_0,s_1, \ldots,s_T\)  visited by \(\mathcal{A}\), the average loss mitigates the effects of irreversible mistakes in the early stages that may dominate the regret.

The expected \(h\)-reachable regret bound \(r^{\cal A}_h(T)\) and average loss bound \(\ell^{\cal A}_h(T)\) are defined as \(r^{\cal A}_{h}(T) \ge V_{\mathcal{L},t,h}^*(s_{0},T) -{V^{\cal A} }(s_{0}, T) \) and \(\ell^{\cal A}_{h}(T)\ge {\frac{1}{T} \sum_{t=1}^{T} \left[V_{\mathcal{L},t,h}^*(s_{t},T) -{V^{\cal A} }(s_{t}, T)\right]}\). That is, they are the same as the  standard expected regret and average loss, respectively, with the exception that the optimal value function \(V^*\) has been replaced by the \(h\)-reachable optimal value function \(V_{\mathcal{L},t,h}^*(s_{t})\).

While the definition of PAC-RMDP(\(h\)) focuses on exploration, the proposed PAC-RMDP(\(h\)) algorithms maintain anytime expected \(h\)-reachable average loss bounds and anytime error bounds, and thus the performances of our algorithms are expected to improve with time, rather than  after some number of exploration steps.

\section{A4. Proofs of Theoretical Results for Algorithm 1}
We  first verify the main properties of Algorithm 1 and then analyze  a practically relevant property of the algorithm in the subsection of Further Discussion. We assume that Algorithm 1 is used with the sample mean estimator, which determines \(\mathcal L\).

\subsubsection{Main Properties} \label{discrete_main_analysis}
To compare the results with those of past studies, we assume that \(R_{max} \le c\) for some fixed constant \(c\). The effect of this assumption can be seen in the proof of Theorem 1. Algorithm 1 requires no input parameter related to \(\epsilon\) and \(\delta\). This is because the required degree of optimism can be determined independently of the unknown aspect of the world. This means that Theorem 1 holds at any time during an execution for a pair of corresponding \(\epsilon\) and \(\delta\).

\begin{lemma}
	(Optimism) For all \(s \in S\) and for all \(t,h \ge 0\), the internal value \(\tilde V^{\mathcal A_t}(s)\) used by Algorithm 1 is at least the \(h\)-reachable optimal value \(V_{\mathcal{L},t,h}^*(s)\); \(\tilde V^{\mathcal A_t}(s) \ge V_{\mathcal{L},t,h}^*(s) \).
\end{lemma}

\begin{proof}
The claim follows directly from the construction of Algorithm 1. It can be verified by induction on each step of the value iteration or the roll-out in a planning algorithm.
\end{proof}

\begin{theorem}
(PAC-RMDP) Let \(\mathcal A_t\) be a policy of Algorithm 1. Let \(z = \max(h, \frac{\ln(2^{|S|}|S||A|/\delta)}{\epsilon(1 - \gamma)})\). Then, for all \(\epsilon>0\), for all $\delta=(0,1)$, and for all \(h\ge0\),
\begin{enumerate}[label={\arabic*)}]
	\item for all but at most \( O\left(\frac{z|S||A|}{\epsilon^{2}(1 - \gamma)^{2}} \ln \frac{|S||A|}{\delta} \right)\) time steps, \( V^{{\cal A}_t}(s_{t}) \geq V_{\mathcal{L},t,h}^*(s_{t}) - \epsilon\), with probability at least \(1 - \delta\), \textit{and}
	\item there exists \(h^*(\epsilon, \delta)=O(\mathcal P(1/\epsilon,1/\delta,1/(1-\gamma),|\text{MDP}|))\) such that \(|V^*(s_t) - V_{\mathcal{L},t,h^*(\epsilon,\delta)}^* (s_{t})| \le \epsilon\) with probability at least \(1 - \delta\).
\end{enumerate}
\end{theorem}

\begin{proof}
Let \(K\) be a set of state-action pairs where the agent has at least \(m\) samples (this corresponds to \textit{the set of known state-action pairs} described by \citeauthor{kearns2002near} \shortcite{kearns2002near}). With the boundary condition \(\overline{V^{\cal A}} (s, 0) = 0\), define the mixed value function \(\overline{ V^{\cal A}}(s, H)\) with a finite horizon \(H' = \frac{1}{1 - \gamma} \ln \frac{6R_{max}}{\epsilon(1 - \gamma)}\) as
\begin{equation*}
\overline {V^{\cal A}}(s, H') =
\begin{cases}
\sum_{s'} P(s'|s, {\cal A}(s)) [R(s, {\cal A}(s), s') + \gamma \overline {V^{\cal A}}(s', H' - 1)] & \text{if }(s, {\cal A}(s)) \in K \\
\max_{\tilde P \in {\cal M}_{\mathcal{L},t,h,(s,a)}} \sum_{s'} \tilde P(s'|s, {\cal A}(s))[R(s, {\cal A}(s), s') + \gamma  \overline {V^{\cal A}}(s', H '- 1)]  & \text{otherwise} \\
\end{cases}
\end{equation*}
	
Let \(A_{K}\) be the escape event in which a pair \((s, a) \notin K\) is generated for the first time when starting at state  \(s_t\), following  policy \({\cal A}_t\), and transitioning based on the true dynamics \(P\) for \(H'\) steps. Then, for all $t,h\ge0$, with probability at least \(1 - \delta/2\),
\begin{dmath*}
V^{\mathcal A_t}(s_t) \ge \overline {V^{\mathcal A_t}}(s_t, H') - \frac{R_{max}}{1 - \gamma} \Pr(A_{k}) - \frac{\epsilon}{6} \ge \tilde V^{\mathcal A_t}(s_t) - \frac{R_{max}}{1 - \gamma} \Pr(A_{k}) - \frac{\epsilon}{3} - \frac{R_{max}}{1 - \gamma} \left(\frac{h}{m} +  \sqrt{ \frac{2\ln(2^{|S| + 1}|S||A|/\delta)}{m}} \right)  \ge V_{\mathcal{L},t,h}^*(s_{t})- \frac{R_{max}}{1 - \gamma} \Pr(A_{k}) - \frac{\epsilon}{3}  - \frac{R_{max}}{1 - \gamma} \left(\frac{h}{m} + \sqrt{ \frac{2\ln(2^{|S| + 1}|S||A|/\delta)}{m}} \right).
\end{dmath*}
	
The first inequality follows from the fact that \(V^{\mathcal A_{t}}(s_{t})\) and \(\overline V^{\mathcal A_t}(s_t)\) are only different when event \(A_K\) occurs, and their difference is bounded above by \(\frac{R_{max}}{1 - \gamma}\) (this is the upper bound on the value \(\tilde V(s_t)\)). Furthermore, the finite horizon approximation adds an error of \(1/6 \epsilon\). A more detailed argument only involves algebraic manipulations that mirror the proofs given by \citeauthor{strehl2008analysis} (\citeyear{strehl2008analysis}, Lemma 3) and \citeauthor{kearns2002near} (\citeyear{kearns2002near}, Lemma 2).
	
The second inequality follows from the fact that \(\overline V^{\cal A}\) is  different from \(\tilde V^{\cal A}\)  only for the state-action pairs \((s, a) \in K\), for which  \(\tilde V^{\mathcal A_t}(s_t)\) deviates from \(\overline V^{\mathcal A_t}(s_t)\) by at most \small \(\frac{R_{max}}{1 - \gamma} (\frac{h}{m} + \sqrt{2\ln(2^{|S| + 1}|S||A|/\delta)/{m}})\) \normalsize with probability at least \(1 - \delta/2\). This is because \(|\tilde V^{\mathcal A_t}(s_t) - V_{\mathcal{L},t,0}^{\mathcal A_t}(s_{t})| \le \frac{R_{max}}{1 - \gamma}\frac{h}{m}\) with certainty, and  \(|V_{\mathcal{L},t,0}^{\mathcal A_t}(s_{t}) -V^{\mathcal A_t}(s_t)| \le \frac{R_{max}}{1 - \gamma}  \sqrt{2\ln(2^{|S| + 1}|S||A|/\delta)/m}\) with probability at least \(1 - \delta/2\) (the later is due to the result of \citeauthor{weissman2003inequalities} (\citeyear{weissman2003inequalities}, Theorem 2.1) and the union bound for state-action pairs).
	
The third inequality follows from Lemma 1.
	
Therefore, if \(h \le \sqrt{2m\ln(2^{|S| + 1}|S||A|/\delta)} \),
\begin{dmath*}
V^{\mathcal A_t}(s_t)  \ge  V_{\mathcal{L},t,h}^*(s_{t}) - \frac{R_{max}}{1-\gamma} \Pr(A_{k}) - \frac{\epsilon}{3} - \frac{2R_{max}}{1 - \gamma} \sqrt{ \frac{2\ln(2^{|S| + 1}|S||A|/\delta)}{m}}.
\end{dmath*}
	
If \(h > \sqrt{2m\ln(2^{|S| + 1}|S||A|/\delta)} \), \[V^{\mathcal A_t}(s_t)  \ge  V_{\mathcal{L},t,h}^*(s_{t})- \frac{R_{max}}{1 - \gamma} \Pr(A_{k}) - \frac{\epsilon}{3} - \frac{2R_{max}}{1 - \gamma} \frac{h}{m}.\]
	
Let us consider the case where \(h \le \sqrt{2m\ln(2^{|S| + 1}|S||A|/\delta)}\). We fix \(m = \frac{72R^2_{max}\ln(2^{|S| + 1}|S||A|/\delta)}{\epsilon^2(1 - \gamma)^2}\) to give \(\frac{\epsilon}{3}\) in the last term on the right-hand side. If \(\Pr(A_K) \le \frac{\epsilon(1 - \gamma)}{3R_{max}} \) for all \(t\), \(V^{\mathcal A_t}(s_t) \ge V_{\mathcal{L},t,h}^*(s_{t}) - \epsilon \) with probability at least \(1 - \delta/2\). For the case where \(\Pr(A_K) > \frac{\epsilon(1 - \gamma)}{3R_{max}} \) for some \(t\), we define an independent random event \(A'_K\) such that \(\Pr(A'_{K}) = \frac{\epsilon(1 - \gamma)}{3R_{max}} < \Pr(A_K) \). According to the Chernoff bound, for all \(k\ge 4\), with probability at least \(1 - \delta/2\), the event \(A_K\) will occur at least \(k\) times after \(\frac{2k}{\Pr(A'_{K})}\ln \frac{2}{\delta}\) time steps. Thus, by applying  the union bound on \(|S|\) and \(|A|\), we have a probability of at least \(1-\delta/2\) of event \(A_K\) occurring at least \(m\) times for all state-action pairs after \(O\left(\frac{m|S||A|}{\Pr(A_k')} \ln \frac{|S||A|}{\delta} \right)  = O\left(\frac{mR_{max}|S||A|}{\epsilon(1 - \gamma)} \ln \frac{|S||A|}{\delta} \right)\) time steps.
	
Let us carefully consider what this means. Whenever \(A_K\) occurs, the sample is used to minimize the error between \( V^{\cal A}\) and \(\tilde V^{\cal A}\) by the definition of \(A_K\). Since \(\tilde V(s) \ge V_{\mathcal{L},t,h}^*(s)\) holds at any time, whenever \(A_K\) occurs, the sample is used to reduce the error in \( V^{\mathcal A_t}(s_t) \ge \tilde V^{\mathcal A_t}(s_t) - \text{(error)}  \ge V_{\mathcal{L},t,h}^*(s_{t}) - \text{(error)}\) (note that if \(\tilde V(s) \ge V_{\mathcal{L},t,h}^*(s)\) holds randomly, this event must occur concurrently with \(A_K\) to reduce the error on the right-hand side). Thus, after this number of time steps, \(Pr(A_K)\) goes to zero with probability at least \(1 - \delta/2\). Hence, from the union bound, the above inequality becomes \(V^{\cal A}(s_t) \ge V_{\mathcal{L},t,h}^*(s_{t}) - \frac{2}{3}\epsilon \) with probability at least \(1 - \delta\).
	
For the case where \(h > \sqrt{2m\ln(2^{|S| + 1}|S||A|/\delta)}\), we fix \(m = \frac{hR_{max}}{6\epsilon(1 - \gamma)}\). The rest of the proof follows that for the case of smaller values of \(h\).
Therefore, we have proved the first part of the statement.
	
Finally, we consider  the second part of the statement.  Let \(\widehat P_{t,h}(\cdot|s,a)\) be the future model obtained by updating the current model \(\widehat P_{}(\cdot|s,a)\) with \(h\) \textit{random} future samples (\(h\) samples   drawn from \(P(S|s,a)\) for each \((s,a)\in (S,A)\)). Using a result given by \citeauthor{weissman2003inequalities} (\citeyear{weissman2003inequalities}, Theorem 2.1),  we know that for all \(s\in S\), with probability at least \(1 - \delta\),
\[\max_{s,a} \|\widehat P_{t, h}(\cdot|s,a) -  P(\cdot|s,a) \|_1 \le  \sqrt{\frac{2\ln(2^{|S| + 1}|S||A|/\delta)}{n_{t,min} + h}},\]
where \(n_{t,min} = \min_{s, a}n_t(s, a)\). Now, if we use the distance function \(d(\widehat P(\cdot|s,a), P(\cdot|s,a))=\|\widehat P(\cdot|s,a) - P(\cdot|s,a) \|_1\) to define the \(h\)-reachable optimal function,
\begin{align*}
|V^*(s_t) - V_{\mathcal{L},t,h^*(\epsilon,\delta)}^{d*} (s_{t})| &\le \frac{R_{max}}{1 - \gamma} \max_{s,a}\|P^{d*}_{\mathcal{L},t,h}(\cdot|s, a) -  P(\cdot|s,a) \|_1 \\
&=  \frac{R_{max}}{1 - \gamma} \max_{s,a} \min_{\widehat P \in{\mathcal{M}}_{\mathcal{L},t,h,(s,a)}}\|\widehat P(\cdot|s, a) -  P(\cdot|s,a) \|_1 \\  &\le\frac{R_{max}}{1 - \gamma}  \sqrt{\frac{2\ln(2^{|S| + 1}|S||A|/\delta)}{n_{t,min} +h}},
\end{align*}
	
The last inequality follows that the models reachable with  \(h\) \textit{random} samples,  \(\widehat P_{t,h}(\cdot|s,a)\) , are contained in a set of \(h\)-reachable models and the best \(h\)-reachable model, \(P^{d*}_{\mathcal{L},t,h}(\cdot|s, a)\),  explicitly minimize the norm, resulting in that \(P^{d*}_{\mathcal{L},t,h}(\cdot|s, a)\) is at least as good as \(\widehat P_{t,h}(\cdot|s,a)\) in terms of the norm. The right-hand side of the above inequality becomes less than or equal to \(\epsilon\) when \(h \leftarrow h^*(\epsilon, \delta) = \frac{2R_{max}^2\ln(2^{|S| + 1}|S||A|/\delta)}{\epsilon^2(1 - \gamma)^2}\). Thus, we have the second part of the statement.
\end{proof}

\begin{corollary}
	(Anytime error bound) With probability at least \(1-\delta\), if \(h\le\frac{\ln(2^{|S|}|S||A|/\delta)}{\epsilon(1 - \gamma)}\),
	\[\epsilon_{t,h} = O\left( \sqrt[3]{ \frac{|S||A|}{t(1-\gamma)^3} \ln \frac{|S||A|}{\delta} \ln \frac{2^{|S|}|S||A|}{\delta} } \right),\]
	and otherwise,
	\[\epsilon_{t,h} = O\left( \sqrt{\frac{h|S||A|}{t(1-\gamma)^2} \ln \frac{|S||A|}{\delta}} \right).\]
\end{corollary}

\begin{proof}
From Theorem 1, if \(t = c \frac{z|S||A|}{\epsilon^{2}(1 - \gamma)^{2}} \ln \frac{|S||A|}{\delta} \) with \(c\) being some fixed constant, \(V^{\cal A}(s_t) \geq V_{\mathcal{L},t,h}^*(s_{t})  - \epsilon\) with probability at least \(1-\delta\). Since this holds for all \(t\ge0\) with corresponding \(\epsilon\) and \(\delta\), it implies that \(\epsilon^2 \le  A \frac{z|S||A|}{t(1 - \gamma)^{2}} \ln \frac{|S||A|}\delta\) with probability at least \(1-\delta\). Substituting \(z= \max(h, \frac{\ln(2^{|S|}|S||A|/\delta)}{\epsilon(1 - \gamma)})\) yields the statement.
\end{proof}

The anytime \(T\)-step average loss is equal to  \(\frac{1}{T}\sum_{t=1}^{T} (1-\gamma^{T+1-t})\epsilon_{t,h,\delta}\). Since the errors considered  in Theorem 1 and Corollary 3  are for an infinite horizon, the factor   \((1-\gamma^{T+1-t})\) translates the infinite horizon error to the \( T\)-step finite horizon error (this can be seen when we modify the proof of Theorem 1 by replacing \(\frac{1}{1-\gamma}\) with \(\frac{1-\gamma^{T+1-t}}{1-\gamma}\)).

\begin{corollary}
	(Explicit exploration runtime) With probability at least \(1-\delta\), the explicit exploration runtime of Algorithm 1 is \( O(\frac{h|S||A|}{\epsilon(1 - \gamma)\Pr[A_K]} \ln \frac{|S||A|}{\delta} )=O(\frac{h|S||A|}{\epsilon^{2}(1 - \gamma)^{2}} \ln \frac{|S||A|}{\delta} )\), where
	\(A_{K}\) is the escape event defined in the proof of Theorem 1.
\end{corollary}

\begin{proof}
The proof directly follows that of Theorem 1 with \(z\). Compared to the sample complexity of Algorithm 1, \(z\) is replaced by \(h\)  based on the proof of Theorem 1.
\end{proof}

\section{A5. Additional Experimental Example for Discrete Domain}
Figure 1 shows the results in the main paper along with 10\% and 90\% values.Aside from the proposed algorithm, only BOLT gathered better rewards than a greedy algorithm while maintaining the claimed theoretical guarantee.

In this example, our proposed algorithm worked well and maintained its theoretical guarantee. One might consider the theoretical guarantee of PAC-RMDP, especially PAC-RMDP(1), to be too weak. Two things should be noted. First, the $1$-reachable value function is not the value function that can be obtained with just one additional sample, but requires an additional sample for all \(|S||A|\) state-action pairs. Second, in contrast to Bayesian optimality, the $1$-reachable value function is not the value function \textit{believed} to be obtained with \(|S||A|\) additional samples, but is \textit{possibly} reachable in terms of the unknown true world dynamics with the new samples.

However, it is  certainly possible to devise a problem such that PAC-RMDP(1) is not guaranteed to conduct sufficient exploration. As an example, we consider a modified version of the five-state chain problem, where the probability of successfully moving away from the initial state is  very small (= 0.05), thus requiring more extensive exploration. We  modified the transition model as follows: Let \(a_1\) be the optimal action that moves the agent away from the initial state. For \(i = \{2, 3, 4, 5\}\), \(\Pr(s_i, a_1, s_{\min(i+ 1, 5)}) = 0.99\) and \(\Pr(s_i, a_1, s_{1}) = 0.01\). For \(i = 1\), \(\Pr(s_i, a_1, s_{i + 1)}) = 0.05\) and \(\Pr(s_i, a_1, s_{1}) = 0.95\). For action \(a_2\) and any \(s_i\), \(\Pr(s_i, a_2, s_{1}) = 1\). The numerical results for this example are shown in Figure 2. As expected, the PAC-RMDP(1) algorithm often became stuck in the initial state.

\begin{figure*}[t!]
	\begin{subfigure}[b]{0.5\linewidth}
		\center
		\includegraphics[trim = 20mm 38mm 22mm 40mm, height=45mm]{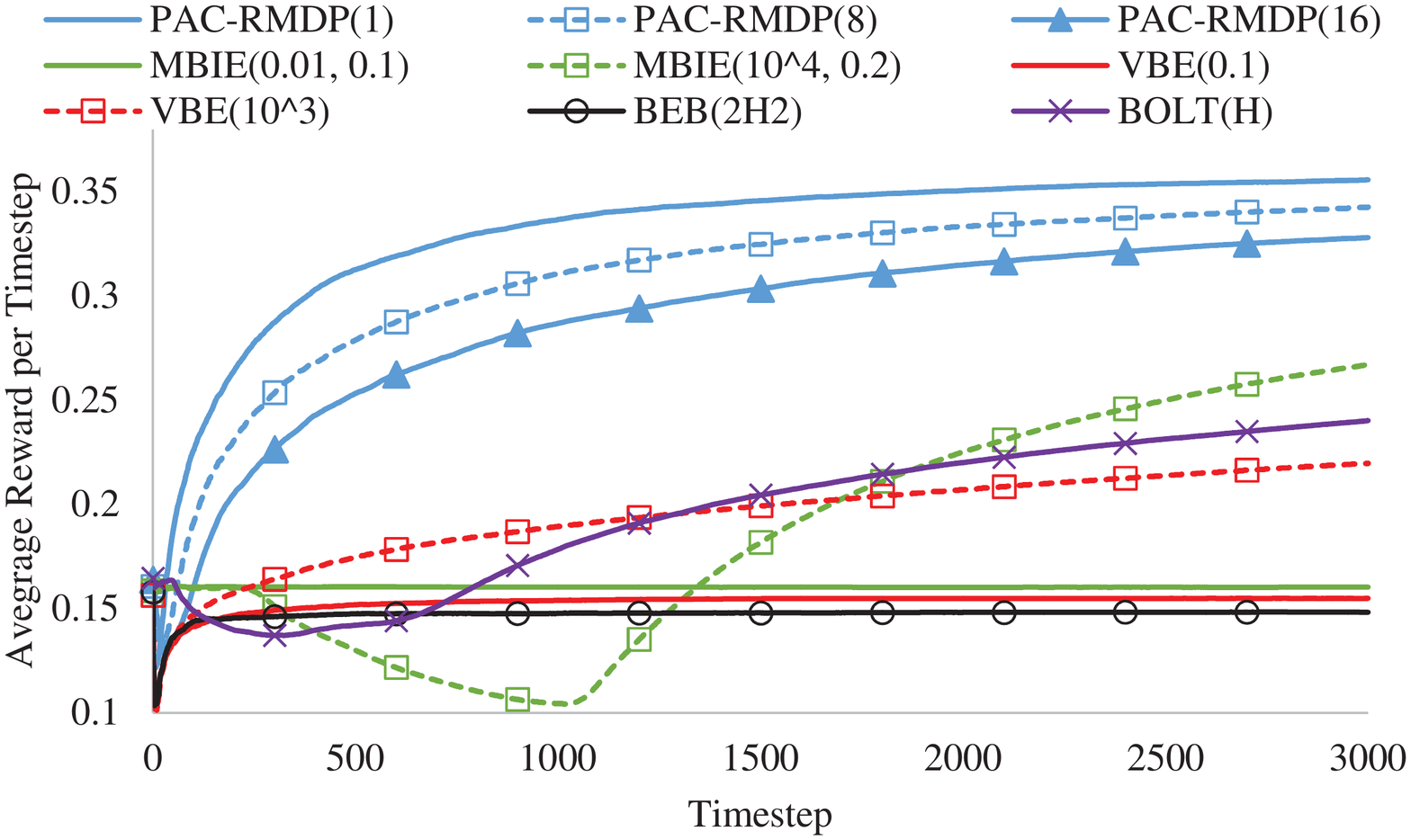}
		\caption{Average of 1000 runs over all time steps}
	\end{subfigure}
	\begin{subfigure}[b]{0.5\linewidth}
		\center
		\small
		\begin{tabular}{| l | c | c | c |}
			\cline{1-4}
			Algorithm & Average & 10\%  & 90\%  \\ \hline
			PAC-RMDP(1) & 0.357& 0.332 & 0.378 \\ \hline
			PAC-RMDP(8) & 0.343 & 0.321 & 0.365 \\ \hline
			PAC-RMDP(16) & 0.328 & 0.305& 0.321 \\ \hline
			MBIE(0.01, 0.1) & 0.160  & 0.158 & 0.162 \\ \hline
			MBIE(20, 0.9) & 0.160 & 0.158 & 0.162 \\ \hline
			MBIE($10^{4}$, 0.2) & 0.267 & 0.250 & 0.285 \\ \hline
			VBE(0.1) & 0.155 & 0.152 & 0.158 \\ \hline
			VBE(0.99) & 0.156 & 0.153 & 0.158 \\ \hline
			VBE($10^{3}$) & 0.220 & 0.207 & 0.232 \\ \hline
			BEB(2$\times148^2$) & 0.148 & 0.142 & 0.154 \\ \hline
			BOLT(148) & 0.240 & 0.221 & 0.256 \\ \hline
		\end{tabular}
		\caption{Results for 1000 runs at time step 3000  }
		\normalsize
	\end{subfigure}
	\caption{Average total reward per time step for the Chain Problem. The algorithm parameters are shown as PAC-RMDP($h$), MBIE($\epsilon, \delta$), VBE($\delta$), BEB($\beta$), and BOLT($\eta$).}
\end{figure*}

\begin{figure*}[t!]
	\begin{subfigure}[b]{0.5\linewidth}
		\center
		\includegraphics[trim = 20mm 38mm 22mm 38mm, height=45mm]{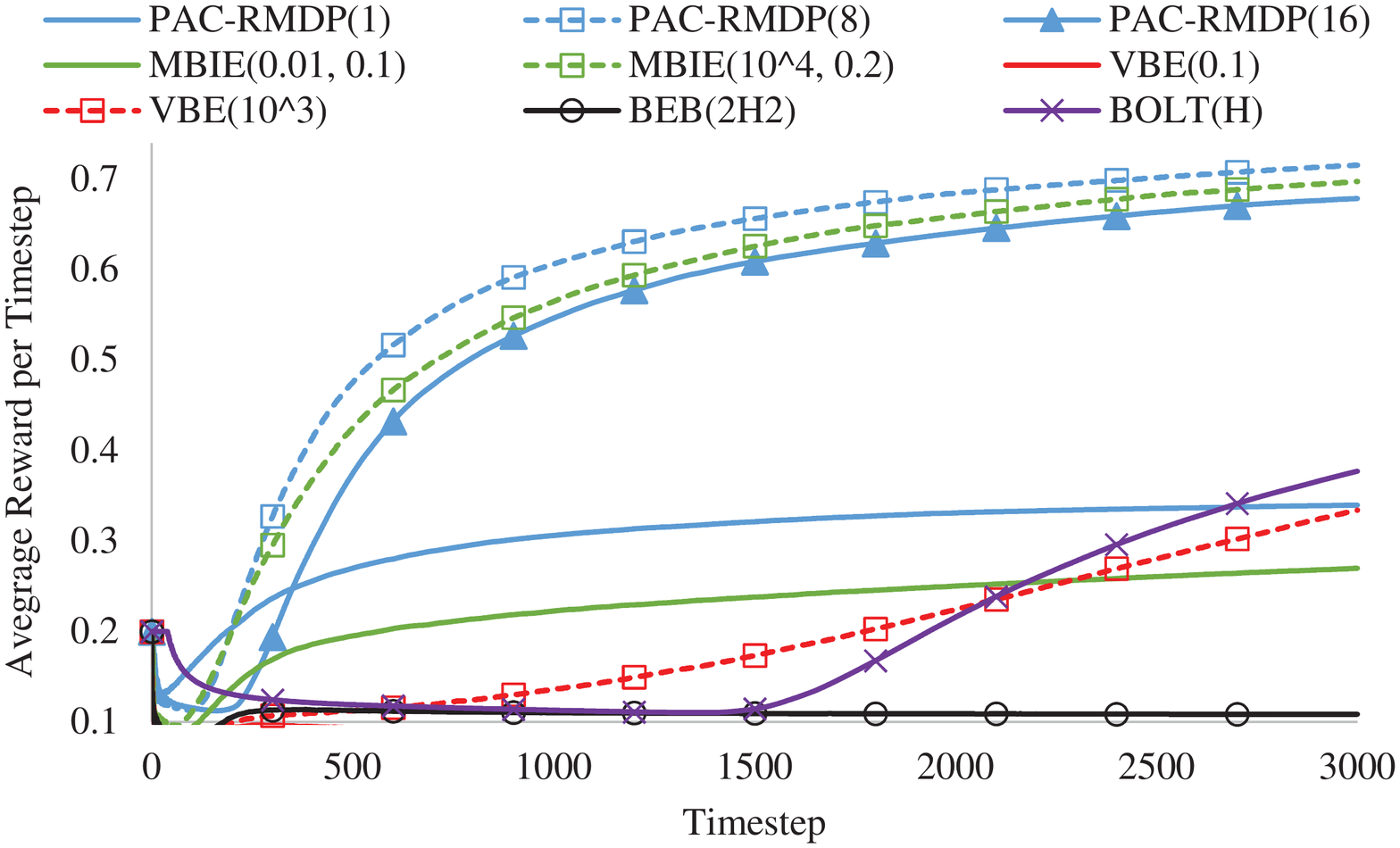}
		\caption{Average for 1000 runs over all time steps}
	\end{subfigure}
	\begin{subfigure}[b]{0.5\linewidth}
		\center
		\small
		\begin{tabular}[b]{| l | c | c | c |}
			\cline{1-4}
			Algorithm & Average & 10\%  & 90\%  \\ \hline
			PAC-RMDP(1) & 0.339& 0.196 & 0.772 \\ \hline
			PAC-RMDP(8) & 0.715 & 0.650 & 0.784 \\ \hline
			PAC-RMDP(16) & 0.678 & 0.612& 0.747 \\ \hline
			MBIE(0.01, 0.1) & 0.270  & 0.260 & 0.279 \\ \hline
			MBIE(20, 0.9) & 0.327 & 0.313 & 0.340\\ \hline
			MBIE($10^{4}$, 0.2) & 0.697 & 0.634 & 0.752 \\ \hline
			VBE(0.1) & 0.090 & 0.060 & 0.122 \\ \hline
			VBE(0.99) & 0.094 & 0.061 & 0.126 \\ \hline
			VBE($10^{3}$) & 0.334 & 0.306 & 0.360 \\ \hline
			BEB(2$\times148^2$) & 0.108 & 0.103 & 0.113 \\ \hline
			BOLT(148) & 0.377 & 0.314 & 0.441 \\ \hline
		\end{tabular}
		\caption{Results for 1000 runs at time step  3000  }
		\normalsize
	\end{subfigure}
	\caption{Average total reward per time step for the modified Chain Problem. The algorithm parameters are shown as PAC-RMDP($h$), MBIE($\epsilon, \delta$), VBE($\delta$), BEB($\beta$), and BOLT($\eta$).}
\end{figure*}

\section{A6. Proofs of Theoretical Results for Algorithm 2}
We assume that Algorithm 2 is used with the least square estimation, which determines \(\mathcal L\). Because the true world dynamics are assumed to have the parametric form \(P(s'|s, a) = {\cal N} (\theta^T\Phi(s, a), \sigma^{2} I)\) with a known \(\sigma\), their unknown aspect is attributed to the weight vector \(\theta\). Therefore, we discuss \(h\)-reachability in terms of \(\hat \theta\) instead of \(\widehat P\). For each \(i^{th}\) component, Let \(\hat \theta_{(i),h,(s, a)}^* \) be the best \(h\)-reachable model parameter corresponding to the best $h$-reachable models, \(\widehat P^*_{\mathcal{L},t,h}\) (we drop the index \(\mathcal L,t\) and \(d\) for brevity); using the set \(\hat \theta_{ (i),h,(s, a)}^*\) for every \((s, a)\) pair results in the \(h\)-reachable value function \(V^{d*}_{\mathcal{L},t,h}\). Note that \(\hat \theta_{(i)}\) is the current model parameter. In the following, we make a relatively strict assumption to simplify the analysis: when they are not provided as  inputs, the estimated values of \(L\) and \(\Delta^{(i)}\) are correct in  that they satisfy  Assumption 2 and $\Delta^{(i)} \ge\sup_{s, a} |(\theta_{(i)} -\hat \theta_{(i)})^T \Phi_{(i)}(s, a)|$. This assumption can be relaxed by allowing  the correctness to be violated with a constant probability. In such a case, we must force the random event to occur concurrently with the escape event, as discussed in the proof of Theorem 1 (the easiest way to do so is to take a union bound over the time steps until convergence). Furthermore, if we can specify the inputs \(L\) and \(\Delta^{(i)}\), there is no need for this assumption.

\begin{lemma}
	(Correctness of the \(h\)-reachable model interval)  For the entire execution of Algorithm 2, for all state components \(1\le i\le n_s\), for all \(t,h\ge 0\), and for all \((s,a)\in (S,A)\), the following inequality holds with probability at least \(1-\delta/2\):
	\[\left|[\hat \theta_{(i)} - \hat \theta_{(i),h,(s, a)}^{*}]^T\Phi_{(i)}(s, a) \right| \le I_h(\Phi_{(i)}(s, a), X_{t}).\]
\end{lemma}

\begin{proof}
Let \(s^*_{1}\in S'_{(s,a)}\) be the future possible observation from which the current model parameter \(\hat \theta_{(i)}\) is updated to \(\hat \theta_{(i),1,(s, a)}^{*}\).  Then,
	\begin{dmath*}
		\left| [\hat \theta_{(i),1,(s, a)}^{*} - \hat \theta_{(i)}]^T\Phi_{(i)}(s, a) \right| = \left| \Phi^T_{(i)}(s, a)(X_t^T X_t)^{- 1}\Phi_{(i)}(s, a) [s^*_{1} - \hat \theta_{(i),1,(s, a)}^{*T}\Phi_{(i)}(s, a)] \right| \le \left| \Phi^T_{(i)}(s, a)D_t(\frac{1}{\lambda _{(1)}}, \ldots , \frac{1}{\lambda _{(n)}}){U_t}^T{\Phi_{(i)}(s, a)}(\Delta^{(i)}+ \varsigma(M) \sigma_{(i)}) \right|.
	\end{dmath*}
	
The first line follows directly from a result given by \citeauthor{cook1977detection} (\citeyear{cook1977detection}, Equation (5)). The second line is due to the following: with probability at least $1-\frac{1}{2}e^{-\varsigma^{2}(M)/2}$,
\begin{align*}
s^*_{1} - \hat \theta_{(i),1,(s, a)}^{*T}\Phi_{(i)}(s, a) \le \theta_{(i)}^{T}\Phi_{(i)}(s, a) - \hat \theta_{(i),1,(s, a)}^{*T}\Phi_{(i)}+ \varsigma(M) \sigma_{(i)}
&\le|\theta_{(i)}^{T}\Phi_{(i)}(s, a) - \hat \theta_{(i),1,(s, a)}^{*T}\Phi_{(i)}(s,a)|+ \varsigma(M) \sigma_{(i)} \\
& \le    |\theta_{(i)}^{T}\Phi_{(i)}(s, a) - \hat \theta_{(i)}^{T}\Phi_{(i)}(s,a)|+ \varsigma(M) \sigma_{(i)} \\
&\le\Delta^{(i)}+ \varsigma(M) \sigma_{(i)}
\end{align*}
where the first inequality follows that \(\Pr(s_{t + 1} > \theta_{(i)}^{T}\Phi_{(i)}(s, a) + \varsigma(M) \sigma_{(i)})<\frac{1}{2}e^{-\varsigma^{2}(M)/2}\) and the third inequality follows the choice of the distance function \(d\) (i.e., the mean prediction with  the best \(h\) reachable model is at least as good as that of the best \(h-1\) model). We then separate the above into two terms with large and small eigenvalues:
with probability at least $1-\frac{1}{2}e^{-\varsigma^{2}(M)/2}$,
	\begin{dmath*}
		\left| [\hat \theta_{(i),1(s,a)}^{*} - \hat \theta_{(i)}]^T\Phi_{(i)}(s, a) \right| \le | \Phi^T_{(i)}(s, a){U_t}{D_t}(\frac{1}{\lambda_{(1)}}, \ldots, \frac{1}{\lambda _{(j)}}, 0, \ldots , 0){U_t}^T{\Phi_{(i)}(s, a) }\left(\Delta^{(i)} + \varsigma(M) \sigma_{(i)}\right) + \Phi^T_{(i)}(s, a){U_t}{D _t}(0, \ldots , 0, \frac{1}{{{\lambda _{(j + 1)}}}}, \ldots , \frac{1}{{{\lambda _{(n)}}}}){U_t}^T\Phi_{(i)}(s, a)(\Delta^{(i)}+ \varsigma(M) \sigma_{(i)}) |.
	\end{dmath*}
	
	With \({w_t}\), we can rewrite part of the second term as \(UD (0, \ldots , 0, \frac{1}{{{\lambda _{(j + 1)}}}}, \ldots, \frac{1}{{{\lambda_{(n)}}}}){U^T} = UD (\frac{1}{{{\lambda _{(1)}}}}, \ldots ,\frac{1}{{{\lambda _{(n)}}}}){U^T}{w_t}.\) Then, with \({g_{t}}\) and \({z_{t}}\), with probability at least $1-\frac{1}{2}e^{-\varsigma^{2}(M)/2}$,
	\begin{dmath*}
		\left| [\hat \theta_{(i),1,(s, a)}^{*} - \hat \theta_{(i)}]^T\Phi_{(i)}(s, a) \right| \le (\Delta^{(i)}+ \varsigma(M) \sigma_{(i)}) \left| {\Phi_{(i)}^T(s, a)g_t}{\Phi_{(i)}(s, a)} + \Phi_{(i)}^T(s, a)z_t {w_t}{\Phi_{(i)}(s, a)} \right|.
	\end{dmath*}
	
	Thus, by applying the union bound for \(h\), with probability at least $1-\frac{h}{2}e^{-\varsigma^{2}(M)/2}$,
	\begin{dmath*}
		\left|[ \hat \theta_{(i),h,(s, a)}^{*} - \hat \theta_{(i)}]^T\Phi_{(i)}(s, a) \right| \le h \left| [\hat \theta_{(i),1,(s,a)}^{*} - \hat \theta_{(i)}]^T\Phi_{(i)}(s, a) \right| \le h (\Delta^{(i)}+ \varsigma(M) \sigma_{(i)}) \left| {\Phi_{(i)}^T(s, a)g_t}{\Phi_{(i)}(s, a)} + \Phi_{(i)}^T(s, a)z_t {w_t}{\Phi_{(i)}(s, a)} \right| \le  I_h(\Phi_{(i)}(s, a), X_{t}).
	\end{dmath*}
	For \(n_s\) components, the above inequality holds with probability at least \(1-\frac{n_sh}{2}e^{-\varsigma^{2}(M)/2}\) (union bound). For all \(M\ge1\), the above inequality holds with probability at least \(1-\frac{n_sh}{2}\sum_{M=1}^{\infty} e^{-\varsigma^{2}(M)/2}\) (union bound). Substituting \(\varsigma(M)=\sqrt{2\ln(\pi^2M^2n_sh/(6\delta))}\), we obtain the statement.
\end{proof}

\noindent In Lemma 3 and Theorem 2, following previous work \cite{strehl2008online,li2011knows}, we assume that an exact planning algorithm is accessible. This assumption will be relaxed by using a planning method that provides an error bound. We also assume that \(R_{max}\le c_{1}\), \(\Delta^{(i)} \leq c_{2}\), and \(\|\theta\| \le c_{3}\) for some fixed constants \(c_{1}, c_2,\) and \(c_3\). Removing this assumption results in these quantities appearing in the sample complexity, but produces no exponential dependence (thus, the sample complexity remains polynomial). We assume that \(M=O(\text{the number of samples})\), meaning that  a planing algorithm calls \(\mathbf I_h\) every iteration at most for a constant number of times. In the following, we use \(\bar n\) to represent the average value of \(\{n_{(1)}, ..., n_{(n_{S})}\}\).
Before analyzing the proposed algorithm, we re-derive the sample complexity of an existing PAC-MDP algorithm \cite{strehl2008online,li2011knows} for our problem setting.

\begin{lemma}
(Sample complexity of PAC-MDP) With an appropriate parameter setting, the PAC-MDP algorithm proposed by \citeauthor{strehl2008online} \shortcite{strehl2008online} and \citeauthor{li2011knows} \shortcite{li2011knows} has the following sample complexity:
\begin{equation*}
\tilde O\left(\frac{{n_S^2{}\bar n^2}}{{{\epsilon ^5}{{(1 - \gamma )}^{10}}}} \right).
\end{equation*}
\end{lemma}

\begin{proof}
The proof follows directly from Theorems 1 and 3  in the previous work of \citeauthor{li2011knows} \shortcite{li2011knows}. The only difference is that we need to take a union bound of different components \(\Phi_{(i)}\) with varying domains, codomains and dimensions \(n_{(s)}\).
\end{proof}

\begin{theorem}
(PAC-RMDP) Let \(\mathcal A_t\) be a policy of Algorithm 2. Let $z = \max(h^2\ln \frac{m^2n_sh}{\delta}, \frac{L^2n_S \bar n\ln^2 m}{\epsilon^3}\ln \frac{n_S}{\delta})$. Then, for all \(\epsilon>0\), for all $\delta=(0,1)$, and for all \(h\ge0\),
\begin{enumerate}[label={\arabic*)}]
	\item for all but at most \(m'=O\left( {\frac{zL^2 n_S\bar n\ln^2 m}{\epsilon^3 (1 - \gamma )^2}}\ln^2\frac{n_S}{\delta} \right)\) time steps (with \(m\le m'\)), \(V^{{\cal A}_t}(s_{t}) \geq V_{\mathcal{L},t,h}^*(s_{t}) - \epsilon\)
	with probability at least \(1 - \delta\), \textit{and}
	\item there exists \small \(h^*(\epsilon, \delta)=O(\mathcal P(1/\epsilon,1/\delta,1/(1-\gamma),|\text{MDP}|))\) such that \( |V^*(s_t) - V_{\mathcal{L},t,h^*(\epsilon,\delta)}^* (s_{t})| \le \epsilon\) with probability at least \(1 - \delta\).
\end{enumerate}
\end{theorem}

\begin{proof}
Let \(\tilde V^A \) be the internal value function used in Algorithm 2. We prove the statement by following the structure of the proof of Theorem 1. Define \(K, m\), \(A_K\), \(\overline V\), and \(H\) in the same manner as in the proof of Theorem 1, and let the vector consisting of \(n_S\) estimation error intervals be \(\mathbf{ER}(s, a) = (|(\theta_{(1)} -\hat \theta_{(1)})^T \Phi_{(1)} (s, a)|, \ldots, |(\theta_{(n_s)} -\hat \theta_{ (n_s)})^T \Phi_{(n_s)} (s, a)|\).  By following the proof of Theorem 1, with probability at least \(1 - \delta/2\) (due to Lemma 2),
	\begin{dmath*}
		V^{\cal A}(s_t) \ge \tilde V^{\cal A}(s_t) - \frac{R_{max}}{1 - \gamma} Pr(A_{k}) - \frac{\epsilon}{3} - L \left(\max_{s, a}\|\mathbf I_h(s, a, X_{m'}) \| + \max_{s, a}\|\mathbf{ER}(s, a) \| \right) \ge V_{\mathcal{L},t,h}^*(s_{t}) - \frac{c_{1}}{1 - \gamma} Pr(A_{k}) - \frac{\epsilon}{3} - L \left(\max_{s, a}\|\mathbf I_h(s, a, X_{m'}) \| + \max_{s, a} \|\mathbf{ER}(s, a) \| \right).
	\end{dmath*}
	
	In the second line, we used the assumption \(R_{max} \le c_{1}\). In the first line, \(\max_{s, a}L \|\mathbf I_h(s, a, X_t) \|\) is the difference between \(\tilde V^{\cal A}(s_t)\) and \(V_{\mathcal{L},t,0}^*(s_{t})\), and \(\max_{s, a}L  \|\mathbf{ER}(s, a) \| \) is the difference between \(V_{\mathcal{L},t,0}^*(s_{t})\) and \(V^{\cal A}\). The second line follows from the fact that \(\tilde V^{\cal A} \ge V_{\mathcal{L},t,h}^*(s_{t})\) because of the correctness of \(I_h\) shown in Lemma 2 and the assignment of the most optimistic value within the interval \(\mathbf I_h\) (based on Assumptions 1 and 2). We now impose an upper bound on \(\|\mathbf I_h(s, a, X_t) \|\) and \(\|\mathbf{ER}(s, a) \|\). Following a proof given by \citeauthor{li2011knows} (\citeyear{li2011knows}, Theorem 1) with the assumption \(\Delta^{(i)} \leq c_{2}\) and \(\|\theta\| \le c_{3}\), with probability at least \(1 - \frac{\delta}{4n_S}\),
	\begin{dmath*}
		\left|(\theta_{(i)} \hiderel{-} \hat \theta_{ (i)})^T \Phi_{(i)} (s, a)\right| \le \|\bar q\| \Delta_E(\hat \theta) + \|\bar u\|  \le \frac{2c_{3}\sqrt{n_{(i)} \ln m}}{m^{1/4}} \left(24c_{2}\ln\frac{8n_S}{\delta}\right)^{1/4} + \frac{(2c_{3}\sqrt{\ln m} + 5)\sqrt{n_{(i)}}}{\sqrt m}
		\le O \left(\frac{(n_{(i)}\ln m)^{1/2}(\ln(n_S/\delta))^{1/4}}{m^{1/4}} \right),
	\end{dmath*}
	where \(\|\bar q\|, \|\bar u\|\) and \(\Delta_E(\hat \theta) \) are as defined by \citeauthor{li2011knows} \shortcite{li2011knows}. Since \(\Phi_{(i)}^Tz_{t}(s_{t + 1} - \hat\theta_{t + 1}^T\Phi_{(i)}) = \hat \theta_{t + 1} - \hat \theta_{t} \), there exist \(\hat \theta\) and \(\hat \theta'\) such that \( \left\| \Phi_{(i)}^T(s, a)z_{t}(\Delta^{(i)}+ \varsigma(M) \sigma_{(i)}) \right\| \le \|\hat \theta - \hat\theta' \| \le  \|\hat \theta\| + \|\hat \theta'\| \le 2{c_{3}}\), where we use the assumption \(\|\theta\| \le c_{3}\). Then, following the proofs of Lemmas 11, 12, and 13 given by \citeauthor{auer2002using} \shortcite{auer2002using},
	\begin{dmath*}
		\frac{I_h(\Phi_{(i)}(s, a), X_{t})}{h} \le (\Delta^{(i)}+ \varsigma(M) \sigma_{(i)})  \sum_{j: \lambda_j \ge 1} \frac{\Phi_{j}^2}{\lambda _j}  + \|\hat \theta - \hat\theta' \|  \sqrt{\sum_{j: \lambda _j < 1} \Phi_{j}^2} \le  \frac{{20(c_2+ \sqrt{2\ln(\pi^2M^2n_sh/(6\delta))} \sigma_{(i)})n\ln (m )}}{{m }} + 2c_{3}\sqrt {\frac{{20n_{(i)}}}{{m }}} \le O \left( \frac{\sqrt{n_{(i)}}}{\sqrt m}\ln m \sqrt{\ln(m^2n_sh/(6\delta))} \right).
	\end{dmath*}
	If \(h \le O(\frac{m^{1/2}(\ln n_S/\delta)^{1/4}}{(\ln m)^{1/2}(\ln(m^2n_sh/(6\delta)))^{1/2}}) \), with probability at least \(1 - n_s \frac{\delta}{4n_s}-\delta/2\),
	\begin{dmath*}
		V^{\cal A}(s_t) \ge V_{\mathcal{L},t,h}^*(s_{t}) - \frac{c_{1}\Pr(A_{k})}{1 - \gamma} - \frac{\epsilon}{3} - O \left(\frac{L{n_S^{1/2}\bar n^{1/2}}(\ln m)^{1/2}(\ln(n_S/\delta))^{1/4}}{m^{1/4}} \right).
	\end{dmath*}
	If \(h > O(\frac{m^{1/2}(\ln n_S/\delta)^{1/4}}{(\ln m)^{1/2}(\ln(m^2n_sh/(6\delta)))^{1/2}})\), with probability at least \(1 - n_s \frac{\delta}{4n_s}-\delta/2\),
	\begin{dmath*}
		V^{\cal A}(s_t) \ge V_{\mathcal{L},t,h}^*(s_{t}) - \frac{c_{1}\Pr(A_{k})}{1 - \gamma} - \frac{\epsilon}{3} - O\left(\frac{Lh{n_S^{1/2}\bar n^{1/2}}}{\sqrt m}\ln m \sqrt{\ln(m^2n_sh/(6\delta))}\right).
	\end{dmath*}
	
	To have \(\epsilon/3\) in the last term, we fix \(m = O(\frac{L^4n_S^2 \bar n^2\ln^4 m}{\epsilon^4}\ln \frac{n_S}{\delta})\ \) for   the former case, and \(m = O(\frac{L^2h^2n_S^{} \bar n\ln^2 m\ln(m^2n_sh/(6\delta))}{\epsilon^2})\) for the latter case. Then, the rest of the first part of the statement  follows from the proof of Theorem 1. That is, we can show that by applying the Chernoff bound, the escape event happens no more than the sample complexity in the statement with probability \(1-\delta/2\) unless the term \(\frac{c_{1}\Pr(A_{k})}{1 - \gamma}\) is negligible.  Taking union bound on the failure probability, we obtain the sample complexity in the statement with probability at leat \(1-\delta\).
	
	Finally, we consider  the second part of the statement, following the proof in Theorem 1.  Let \(\hat \theta_{(i),h,(s,a)}\) be the future model parameter obtained by updating the current model \(\hat \theta_{(i)}\) with \(h\) \textit{random} future samples (\(h\) samples drawn from \(P(S|s,a)\) for each \((s,a)\in (S,A)\)). Based on the first part of the proof, \(|(\theta_{(i)} - \hat \theta_{(i),h,(s,a)})^T \Phi_{(i)} (s, a)| \le O \left(\frac{(n_{(i)}\ln (n_{min} + h))^{1/2}(\ln(n_S/\delta))^{1/4}}{(n_{min} + h)^{1/4}} \right)\) with probability at least \(1 - \delta\). Since \(|(\theta_{(i)} - \hat \theta_{(i),h,(s, a)}^{*})^T \Phi_{(i)} (s, a)|\le |(\theta_{(i)} - \hat \theta_{(i),h,(s,a)})^T \Phi_{(i)} (s, a)|\) (this directly follows the definition of \( \hat \theta_{(i),h,(s, a)}^{*}\) and the choice of the distance function \(d\)), this implies that \(h^*(\epsilon, \delta) = O(\frac{L^4n_S^2 \bar n^2\ln^2 m}{\epsilon^4}\ln \frac{n_S}{\delta})\) is sufficient.
\end{proof}

\begin{corollary}
	(Anytime error bound) With probability at least \(1-\delta\), if \(h^2\ln \frac{m^2n_sh}{\delta}\le \frac{L^2n_S \bar n\ln^2 m}{\epsilon^3}\ln \frac{n_S}{\delta}\),
	\[\epsilon_{t,h} = O\left( \sqrt[5]{ {\frac{L^{4} n_S^{2}\bar n^{2}\ln^{2} m}{t(1 - \gamma )}}\ln^3\frac{n_S}{\delta} } \right),\]
	and otherwise,
	\[\epsilon_{t,h} = O\left( {\frac{h^2L ^2n_S^{}\bar n^{}\ln^2 m}{t(1 - \gamma )}}\ln^2\frac{n_S}{\delta} \right).\]
\end{corollary}

\begin{proof}
The proof follows directly from Theorem 2 and the proof of Corollary 1.
\end{proof}

As in the discrete case, the anytime \(T\)-step average loss can be  computed by  summing  the anytime errors as \(\frac{1}{T}\sum_{t=1}^{T} (1-\gamma^{T+1-t})\epsilon_{t,h,\delta}\). In addition, we can derive the explicit exploration runtime.

\setcounter{corollary}{5}
\begin{corollary}
	(Explicit exploration runtime) With probability at least \(1-\delta\), the explicit exploration runtime of Algorithm 2 is
	\begin{dmath*}
		O\left({\frac{h^2L^2 n_S^{}\bar n\ln m}{\epsilon^2\Pr[A_k]  }}\ln^2\frac{n_S}{\delta} \ln \frac{m^2n_sh}{\delta}\right)=O\left({\frac{h^2L^2 n_S^{}\bar n\ln m}{\epsilon^{3}(1-\gamma)  }}\ln^2\frac{n_S}{\delta} \ln \frac{m^2n_sh}{\delta} \right),
	\end{dmath*}
	where \(A_{K}\) is the escape event defined in the proof of Theorem 2.
\end{corollary}

\begin{proof}
	The proof follows that of Theorem 2. Compared to the sample complexity of Algorithm 2, \(z\) is replaced by \(h\)  based on the proof of Theorem 2.
\end{proof}

\section{A7. Experimental Settings for Continuous Domain}
For each problem used in the main paper, we present more detailed descriptions of the experimental settings.

\subsubsection{Mountain Car}
In the mountain car problem, the reward is negative everywhere except at the goal. To reach the goal, the agent must first travel far away, and must explore the world to learn this mechanism. To require a greater degree of exploration, we modify the original problem as follows: The agent obtains a reward equal to -0.9 around the initial position (position = [-0.6, 0.4]), and -1.0 everywhere else but at the goal. At the start of each episode, the agent is always at the bottom of the valley (position = -0.5) with zero velocity. Moreover, a small amount of Gaussian noise with standard deviation 0.001 is added to the velocity. Our model uses 10 grids of residual basis functions over the control signal and velocity as features. For the planning phase, we use a fitted value iteration with a \(30 \times 30\) grid of residual basis functions.  We set \(\Delta^{(i)}\) and the corresponding parameter in the PAC-MDP algorithm to be 0.14, because the velocity is bounded in \([-0.07,0.07]\). Each episode consists of 2000 steps, and we conduct simulations for 100 episodes.

\subsubsection{Simulated HIV treatment}
This problem is described by a set of six ordinary differential equations \cite{ernst2006clinical}. An action corresponds to whether the agent administers two treatments (RTIs and PIs) to patients (thus, there are four actions). Two types of exploration are required: one to learn the effect of using treatments on viruses, and another to learn the effect of not using treatments on immune systems. Learning the former is necessary to reduce the population of viruses, but the latter is required to prevent the overuse of treatments, which weakens the immune system. We select the initial state to be unhealthy, following \citeauthor{ernst2006clinical} \shortcite{ernst2006clinical} and \citeauthor{pazis2013pac} \shortcite{pazis2013pac}. As in previous work, we assume that \textit{noise-free} data can be obtained every five days. Unlike past studies, we assume that \textit{noisy} data can be obtained a day after each instance of noise-free data is collected, with the noise term being \(\zeta'\sim \mathcal{N}(0, 0.1)\). We add another noise term to represent the model error with \(\zeta \sim \mathcal{N}(0, 0.01)\) for each dynamic state. For the model, we use  the six states and the multiple of any two of these six states as features (i.e., the number of features is \(6 + {6 \choose 2} \)). For planning, we use a greedy roll-out method, as described by \citeauthor{adams2004dynamic} (\citeyear{adams2004dynamic}, Section 5).  We set \(\Delta^{(i)}\) and the corresponding parameter in the PAC-MDP algorithm to be the average error among all the predictions and observations. Each episode consists of 1000 days, and we conduct simulations for 30 episodes.

\bibliographystyle{aaai}
\bibliography{kawaguchi}

\end{document}